\documentclass[onecolumn,journal,draftcls,a4paper,12pt]{IEEEtran}
\IEEEoverridecommandlockouts

\usepackage[utf8]{inputenc} 
\usepackage[T1]{fontenc}
\usepackage{url}
\usepackage{ifthen}
\usepackage{cite}
\usepackage[cmex10]{amsmath} 
\usepackage{tikz,tikz-cd}
\usepackage{cite}
\usepackage{amsmath,amssymb,amsfonts}
\usepackage[noend]{algpseudocode}
\usepackage{algorithmicx,algorithm}
\usepackage{graphicx}
\usepackage{textcomp}
\usepackage{xcolor}
\usepackage{bbm}
\usepackage{mathrsfs}
\usepackage{amsthm}
\usepackage{subfigure}
\usepackage{threeparttable}

\newtheorem{definition}{Definition}

\newtheorem{theorem}{Theorem}

\newtheorem{fact}{Fact}
\newtheorem{corollary}{Corollary}
\newtheorem{remark}{Remark}

\begin{document}
\title{Privacy-Preserving Communication-Efficient Federated Multi-Armed Bandits  }

\author{\IEEEauthorblockN{Tan Li, Linqi Song }\\
\IEEEauthorblockA{City University of Hong Kong, \\City University of Hong Kong Shenzhen Research Institute}	 \\		
\{tanli6-c@my., linqi.song@\}cityu.edu.hk
}

\maketitle

\begin{abstract}
Communication bottleneck and data privacy are two critical concerns in federated multi-armed bandit (MAB) problems, such as situations in decision-making and recommendations of connected vehicles via wireless. In this paper, we design the privacy-preserving communication-efficient algorithm in such problems and study the interactions among privacy, communication and learning performance in terms of the regret. To be specific, we design privacy-preserving learning algorithms and communication protocols and derive the learning regret when networked private agents are performing online bandit learning in a master-worker, a decentralized and a hybrid structure. Our bandit learning algorithms are based on epoch-wise sub-optimal arm eliminations at each agent and agents exchange learning knowledge with the server/each other at the end of each epoch. Furthermore, we adopt the differential privacy (DP) approach to protect the data privacy at each agent when exchanging information; and we curtail communication costs by making less frequent communications with fewer agents participation. By analyzing the regret of our proposed algorithmic framework in the master-worker, decentralized and hybrid structures, we theoretically show trade-offs between regret and communication costs/privacy. Finally, we empirically show these trade-offs which are consistent with our theoretical analysis. 
\end{abstract}


\section{Introduction}

Federated multi-armed bandits (MAB), combining the conventional MAB model and federated learning, is an emerging framework in distributed sequential decision-making, especially for many real-world wireless distributed systems \cite{li2020federated} \cite{Shi_Shen_2021}. For example, consider wireless coupon distribution systems in chain stores that make local recommendations to their customers; they wish to aggregate the overall responses without revealing users' personal information to provide better recommendations  \cite{song2017making} \cite{song2018itw} \cite{8849556}. Another application example is the Internet of Vehicles (IoV), where we need to effectively collect information from each vehicle for joint path planning when communication resources are limited  \cite{lin2019uav} \cite{xu2020collaborative}.

Two critical concerns of federated MAB are the communication bottleneck and data privacy. A substantial number of local agents/devices periodically exchanging local model updates for model aggregation require a large amount of communication resources \cite{mills2019communication}, which is often scarce in wireless networks. Several communication efficient learning techniques have been studied recently, including local model compression \cite{haddadpour2021federated},
partial device participation \cite{konevcny2016federated}, and less frequent aggregation \cite{reisizadeh2020fedpaq}. On the other hand, exchanging learning model information while protecting data privacy at local agents involves privacy mechanism design, e.g., differential privacy \cite{abadi2016deep}. Some existing works have began studying how communication constraints and privacy requirements affect the learning performance in supervised learning \cite{hu2020cpfed}, however, these works did not look into this question in the federated bandit problem. 

In this paper, we explore \textit{how we design federated bandit algorithms and communication protocols to better inform what decisions to make at the distributed agents to achieve a certain level of data privacy under communication constraints.}

We consider different network structures for agents to exchange learning knowledge: a master-worker, a decentralized, and a hybrid structure that combines the above two. 

In the master-worker network structure, a central server collects $M$ agents' individual model parameters and returns back an aggregated result to all agents. We propose an elimination-based federated bandit algorithm and an epoch-based communication protocol. In this approach, an exponentially increasing number of time slots consist of an epoch and agents locally explore a set of active arms (out of total $K$ arms) during each epoch. At the end of each epoch, local agents send the estimated mean of each arm with additional noise sampled from a Laplace distribution with parameter $\epsilon$, in order to fulfill the differential privacy requirement. The server reduces the size of this active set by eliminating empirically inferior arms based on the aggregation result and returns elimination results to agents. \textcolor{black}{Our CDP-MAB algorithm achieves a regret of $ O(\max\{\frac{K\log(KT\log T)}{\Delta}, \frac{K\log T \sqrt{\log(KT\log T)}}{\sqrt{M}\epsilon}\}$ with a communication cost of no more than $O(c_1M\log T)$, where $c_1$ is the cost of building a server-to-agent link. This indicates a trade-off between the privacy and learning performance, the second term $\frac{K\log T \sqrt{\log(KT\log T)}}{\sqrt{M}\epsilon}$ in the regret increases with the decreasing of $\epsilon$ (higher level of privacy), i.e., inversely proportional to $\epsilon$.}

Furthermore, when communication constraints are posed, we investigate an efficient communication protocol that allows only a certain number of communication epochs $R$ and at each epoch allows only a fraction $p$ of users to upload learning knowledge. Again, we consider an exponentially increasing length of epochs, but scaled according to $R$. In this case, the server needs to fine-tune the elimination threshold based on insufficient collections of learning knowledge in order to gain a certain level of confidence to remove an empirically inferior arm. This algorithm achieves $\min\{M,T^{2/R}/p^{3/2}\}$ times more regret than the non-communication-constrained setting with a communication cost of $c_1pMR$. 

We next extend to a decentralized network structure where agents could exchange protected learning knowledge with their neighbors to make better decisions. We consider the communication network as an undirected graph $G(V,E)$ with vertices corresponding to the agents and edges depicting neighbor relationships. \textcolor{black}{We propose a multi-hop information propagation protocol, termed Global Information Synchronization (GIS)  protocol, to ensure that all agents can receive messages from other $M-1$ agents after each communication round, however, with a certain transmission delay depending on the network structure. Before the private information is fully synchronized, all agents exploit their locally observed best arm so far. Compared with the centralized setting, there is an additional term $O(Md_G)$ in the decentralized regret, which can be seen as caused by the  information dissemination in the network graph $G$ with diameter $d_G$}.

  
In addition, we consider a hybrid network structure combining master-worker and decentralized network together with a two-layer communication protocol. Each agent first performs local exploration and sends the protected means to a ``sink agent" inside a component (local communication). After this first-step information exchange, the second-step communication only occurs among sink agents of each component and a server (global communication). The server aggregates protected learning parameters and sends global elimination results to each agent. \textcolor{black}{Results show that hybrid structure can help to achieve communication efficiency without deteriorating the regret, by reducing both agent-to-agent and server-to-agent links}. 

Finally, we empirically show trade-offs between learning regret and communication/privacy, which are consistent with our theoretical findings. These results also provide important insights into designing practical communication-efficient privacy-aware federated MAB systems.

\section{Related Work}
The MAB model is widely used in many applications, like recommendation systems and clinic trials, due to its simplicity and efficiency \cite{li2010contextual}\cite{zeng2016online}. Recently, privacy issues have raised concerns in the bandit studies. Early works focus on single-agent MAB problems, where several differential privacy-based bandit learning algorithms have been proposed by adding noise to partial sums of rewards \cite{tossou2016algorithms} \cite{agarwal2017price}. Given a certain level of differential privacy requirement, lower bounds of learning performance in terms of the regret have been given in \cite{ren2020multi} \cite{basu2019differential}.

{\color{black}
There has been another line of literature recently that talks about the distributed/federated MAB problems, where agents aim to collaboratively make decisions by exchanging information with others \cite{agarwal2021multi}. 
Two critical concerns of federated MAB are the communication bottleneck and data privacy. 
In our previous conference work \cite{li2020federated}, the MAB problem has been extended into a multi-agent setting where both a `master-worker' and a fully decentralized structure are studied together with a tree-based privacy preserving mechanism. In \cite{zhu2021federated,sankararaman2019social,martinez2019decentralized}, federated linear bandit problem is investigated through a decentralized network of agents via privacy preserving gossip approach. On the one hand, these works add Laplacian noise to the local estimated means at each time slot before communication, which leads to an $O(K\log^{2.5}(T))$ order privacy-related regret. Instead, we use a simple and efficient privacy-preserving mechanism to scale this term as $O(\log^{1.5} T)$. On the other hand, these works did not take into account of communication efficiency in the protocols. They force the agents to communication to (one of) their neighbors at each time slot to reach global consensus in finite time slots, which incurs $O(T)$ communication cost. 

Federated MAB under limited communication has received more attention recently \cite{agarwal2021multi}. The most common way is to achieve through the central server. In the master-worker structure, a number of agents periodically upload local parameters to the server to reduce the communication rounds. Specifically, in \cite{shi2021federatedpersonal}, a mixture bandit model to balance the generalization and personalization is studied. In \cite{Shi_Shen_2021}, a federated bandit problem with client sampling is studied with communication cost counted in the regret, yet without privacy preserving mechanisms. They implicitly bound the total communication rounds by $O(\log T)$ using an action elimination-based algorithm. Different from them, we study how the regret will be affected when the communication resource is clearly constrained, for example, when the communication round is fixed. In addition, we also proved that under the proposed decentralized and hybrid structure, the $O(\log T)$ regret can be achieved with $O(\log T)$ communication cost.

}

\section{System Model and Problem Formulation}
\label{sec:formulation} 
\subsection{Federated Multi-Armed Bandit Framework}

We consider a federated bandit problem with $M$ agents in either master-worker or decentralized network structures or both. In the master-worker structure, agents communicate with a central server to train a model together; in the decentralized network structure, agents communicate directly with their neighbours via a given network structure and train local models at each agent.

All $M$ agents are associated with $K$ arms (e.g., movies, ads, news, or items) from an arm set $\mathcal{A}: =\left\lbrace  1,2,...,K\right\rbrace $.  At time slot $t$, each agent $i$ chooses to pull an arm $a_i(t)\in \mathcal{A}$. Then the arm $k\in \mathcal{A}$ chosen by agent $i \in [M]$ generates an i.i.d. reward $r_{i,k}(t)\in [0,1]$ from a fixed but unknown distribution at time $t$. We denote by $\mu_{i,k}$ the unknown mean of reward distribution. In our model, we first assume a \textit{homogeneous} reward structure, that for all arms $1\leq k \leq K$, $\mu_{1,k}=\mu_{2,k}=\cdots=\mu_{M,k}$, and thus in the rest of the paper we use $\mu_k$ for simplicity. Without loss of generality, we assume that $\mu_1$ is the best arm. Then the suboptimality gap can be defined as $\Delta_{k} : = \mu_{1} - \mu_{k}$ for any arm $k\neq1$. We also denote by $\Delta$ the minimal non-zero suboptimality gap among all $\Delta_{k}$.
\textcolor{black}{In Section VI, we extend our setting to the \textit{heterogeneous} reward structure, where $\mu_{i,k}=\mu_{j,k}$ ($i\not=j$) does not necessarily hold and agents aim to learn $\mu_k \triangleq \sum_{i\in [M]} \mu_{i,k}/M$, and discuss its theoretical results.}

The objective of the $M$ agents is to minimize the regret, which is defined as the expected reward difference between the best arm and the online learning policies of the agents as follows:. 
\begin{equation}
\mathit{R}(T) = TM\mu_1-\mathbb{E}[\sum_{t=1}^{T}\sum_{i=1}^{M}r_{i,a_i(t)}(t)],
\end{equation} 
where the expectation is taken over the randomness in the choice of arms.

\subsection{Communication Structures and Cost}
In this section, we first talk in detail the communication network structures in federated bandits, leading to different ways of information exchanging. 

$\bullet$ {\bf Master-worker structure.} In this structure, individual agents first perform local learning (pulling some arms) according to some privacy-preserving learning strategies, and then upload protected model parameters to a central server. The central server aggregates information sent from local agents and sends the global parameters back to all agents. This will be conducted iteratively.

$\bullet$ {\bf Decentralized network structure.} This decentralized network structure is described by an undirected, connected graph $G(V,E)$, where $V$ is the set of all $M$ agents and $E$ is the set of all communication links. Agents learn locally and communicate with their neighbors iteratively to exchange their learned knowledge.

$\bullet$ {\bf Hybrid network structure.} We also consider a hybrid structure, where agents first form decentralized network structures, called components, and then components are connected with a central server. The information exchange consists of two levels: local communication where agents within a same component exchange information with the assigned ``sink agent" and global communication where components upload information to the central server and the central server sends back aggregated information to all agents.  

In our framework, we consider that communication is constrained due to the scarce of communication resources so that communication-efficient learning strategies are needed. In particular, \textcolor{black}{we consider the communication cost $C(T)$ to be the cost of building total number of (two-way) communication links when agents exchange information up to time horizon $T$}. 



{\color{black}Consider that the communication happens for a total of $R$ rounds. The $r$-th communication round contains $t_r$ time slots. At each time slot $1\leq t \leq t_r$ in the $r$-th round, $L_{r,t}$ communication links are built and each of them incurs a cost $c_{r,t,l}$ for $1\leq l\leq L_{r,t}$. Then, we have the communication cost over time horizon $T$ as:
\begin{equation}\label{cost_def}
 C(T) = \sum_{r=1}^R \sum_{t=1}^{t_r}\sum_{l =1}^{L_{r,t}}c_{r,t,l}.   
\end{equation}

Refer to the three communication structures we introduced above. We assume that the cost required to build a server-to-agent link is $c_1$, while the cost required to establish an agent-to-agent connection is $c_2$. Therefore, $c_{r,t,l}$ in Eq. (\ref{cost_def}) can take the value of either $c_1$ or $c_2$. In general, the communication cost is related to the energy consumption, available bandwidths, etc. to build the links. 
For example, in the decentralized structure, an agent-to-agent link is established in the local area network within a short distance and with low energy-consumption. The central server in the master-worker structure is deployed in the remote cloud or an edge computing platform, and it consumes more energy and bandwidth resources to establish a server-to-agent link. In this paper, we reasonably assume $c_1>>c_2$ for the hybrid structure where we first consider the peer-to-peer communication and then the peer-server communication.}

Next, we look into details of the information that is being transmitted from an agent. Let us denote by $H_i(t) = (a_i(1),r_{i,a_i(1)}(1),...,a_i(t),r_{i,a_i(t)}(t))$ the historical action-reward pairs observed by agent $i$ until time $t$. Consider at some time slot $t$ when communication happens, agent $i$ creates a message $\mathbf{I}_i(t) = f(H_i(t))$ and sends it to the server or its neighbors, where $f()$ is a privacy-preserving function of agent $i$'s history, such as a noised version of empirical reward mean or cumulative sampling numbers.

\subsection{Differential Privacy Guarantee}\label{DP}

In the federated bandit setting, we consider a privacy model that aims to protect the private historical data when revealing messages $\mathbf{I}_{i}(\cdot)$ that are sent out by each agent. 

\begin{definition}[$\epsilon$- differential privacy in federated bandits]
A mechanism $f(H_i(t))$ is $\epsilon$- differentially private if for all time $t$ when communication occurs, any adjacent histories $H_i(t),  H'_i(t)$, and any measurable subset $O$ of images, we have
\begin{eqnarray}
\Pr\{f(H_i(t))\in O\}\leq \Pr\{f(H'_i(t))\in O\}e^{\epsilon}
\end{eqnarray}
where the adjacent sequences $H_i(t),  H'_i(t)$ differ in at most one position.
\end{definition}

We can see that $\epsilon$- differentially private mechanism $f()$ in federated bandits that generates each message ${\bf I}_i(t)$ is individually private, regardless of the receiver’ algorithm. We next introduce the Laplace mechanism that convert some function $g(H_i(t))$ into a differentially private mechanism $f(H_i(t))$ by adding some Laplace noise  $Laplace ()$ according to the Laplace distribution (which can be multi-dimensional with each dimension being generated i.i.d.).
  
\begin{definition} [See \cite{chan2011private}]
The Laplace mechanism is defined as:
\begin{equation}
f(H_i(t)) = g(H_i(t)) + Laplace (\frac{s(g)}{\epsilon}),
\end{equation}
where $s(g)$ is the sensitivity function under the $l_1$ norm:
\begin{equation}
s(g) = \max _{H_i(t), H'_i(t)}|| g(H_i(t))-g(H'_i(t))||_1,
\end{equation}
which gives an upper bound on how much we must perturb its output in order to preserve privacy. 
\end{definition}

{\bf Our design goal is to design differentially private bandit learning algorithms in the federated bandit framework under communication constraints.}

\section{ Centralized Federated Multi-Armed Bandits}

In this section, we present the communication-efficient privacy-preserving algorithms for federated MAB in the master-worker structure, where a server collects $M$ agents' individual model parameters and returns back an aggregated result to all agents. Our algorithms need to consider: 1) \textit{When and what to communicate?} That is, how to design the communication protocol to balance the exploration-exploitation dilemma in federated bandit learning?  2) \textit{How to protect the privacy of messages when communication happens?} More specifically, how much noise should be added to the parameters to achieve the desired privacy level? Burying these two questions in mind, we first design a federated bandit algorithm that considers privacy requirements with sufficient communications, and then develop a federated algorithm under communication constraints.

 \subsection {Centralized Differentially Private Multi-Armed Bandit Algorithm (CDP-MAB) }
 The CDP-MAB algorithm is described in Alg.~\ref{alg:cdp}. 
 The server maintains an active arm set $I^{(\cdot)}$, initialized as $I^{(0)} = [K]$, and uses an elimination method to gradually eliminate suboptimal arms while learning the optimal arm. The algorithm operates in epochs and each of them can be divided into two sub-phases:
  
\subsubsection{Local Exploration} In epoch {$r$}, all agents receive a new arm set { $I^{(r-1)}$ } broadcast by the server. \textcolor{black}{Agents then explore all the active arms in this set for the same number of $S(r) - S(r-1)$ times (Line 4) and update the empirical means based on observed rewards. Specifically, $S(r)$ is doubling-increasing, so $S(r) - S(r-1)$ is also doubling-increasing}. 

\subsubsection{Communication and Aggregation}: 
After $S(r) - S(r-1)$ times of local explorations, the empirical mean of an active arm $k$ at agent $i$ is denoted by $\hat{x}_{i,k}(r)$, which is only calculated based on the reward at round $r$. \textcolor{black}{Each agent first transfers $\hat{x}_{i,k}(r)$ to the privacy-preserving or protected version $\hat{y}_{i,k}(r)$ with additional noise sampled from the distribution $Lap(\frac{1}{M\epsilon [S(r) - S(r-1)]})$. In order to use the previous rewards without revealing privacy, we introduce $\bar{y}_{i,k}(r)$ to record the private mean of history,
\begin{align}
    \bar{y}_{i,k}(r) &= \frac{S(r-1)}{S(r)}\bar{y}_{i,k}(r - 1) + \frac{S(r)-S(r-1)}{S(r)}\hat{y}_{i,k}(r)
\end{align}}
 Then, all $M$ workers upload the protected mean $\bar{y}_{i,k}(r),  \forall k \in I^{(r-1)}$ to the server. The server aggregates the means  and privately eliminates suboptimal arms based on the confidence interval $C(r)$ (Line 5).  If there exists only one arm (empirically best arm) in the active arm set, then all agents just play this arm until time $T$.

  \begin{algorithm}[H]
 	\caption{\textcolor{black}{Centralized Differentially Private Multi-armed Bandit Algorithm (CDP-MAB)}} 
 	\label{alg:cdp}
 	{\bf Input:}
 	Time horizon: $T$;  Privacy parameter $\epsilon$; number of agent $M$;\\ 
 	{\bf Initialization:} 
 	$t = 1$, $r=1$; $I^{(0)} = [K]$; $x_{i,k}(1) = \hat{y}_{i,k}(1) = \bar{y}_{i,k}(1) = 0$; $S(0) = 0$
 	\begin{algorithmic}[1]
 	    \While{$t\leq T$}
 		\While {$|I^{(r-1)}|>1$}
 		\State Set $\widetilde{\Delta} _r\leftarrow2^{-r}$
 		\State Set $S(r) \leftarrow
 		 \max\{\frac{8\log(8|I^{(r-1)}|r^2T)}{M\widetilde{\Delta}_r^2 }, \frac{8r \sqrt{2\log(8Kr^2T)} }{M^{1.5}\epsilon\widetilde{\Delta}_r}\}$
 		 \State Set $C(r) \leftarrow  \sqrt{\frac{\log (8|I^{(r-1)}|r^2T)}{2MS(r)}} +\frac{r\sqrt{8\log(8Kr^2T)}}{M^{1.5}\epsilon S(r)}$
 		\For{each agent $i = 1,...,M$ }
 		\State \textit{Local exploration:} choose arm  $k \in I^{(r-1)}$ for $\{S(r) - S(r-1)\}$ times;
 		\State Update local empirical mean $\hat{x}_{i,k}(r)$;
 		\State Calculate the privacy mean $\hat{y}_{i,k}(r) = \hat{x}_{i,k}(r) +  Lap(\frac{1}{M\epsilon [S(r) - S(r-1)]})$;
 		\State Calculate the historical mean: $\bar{y}_{i,k}(r) = \frac{S(r-1)}{S(r)}\bar{y}_{i,k}(r - 1) + \frac{S(r)-S(r-1)}{S(r)}\hat{y}_{i,k}(r)$;
 		\State \textit{Communication}: upload $\bar{y}_{i,k}(r)$ to server;
 		\EndFor	
 		\State \textit{Global aggregation} at the server: $\bar{y}_k(r)=\frac{1}{M}\sum_{i=1}^M \bar{y}_{i,k}(r)$;
 		\State Let $\bar{y}_{max}(r) = \max_{k \in I^{(r-1)}}\bar{y}_k(r)$;
 		\If {$\bar{y}_{max}(r) -\bar{y}_k((r) )\geq 2C(r)$}
 		\State \textit{Global elimination} at the server: update $ I^{(r)} = I^{(r-1)}/\{k\}$.  
 		\EndIf
       	\State $r = r+ 1, t = t+|I^{(r-1)}|[S(r) - S(r-1)]$  
 		\EndWhile
 		\State All agents pull the arm until time $T$.
 		\EndWhile

 	\end{algorithmic}
 \end{algorithm}
 
 The performance of the CDP-MAB algorithm is shown in Theorem~\ref{thm:cdp}. 
\begin{theorem}[Performance of the CDP-MAB]
\label{thm:cdp}
	Given time horizon $T$, and privacy level $\epsilon$ and the cost $c_1$ for building a server-to-agent link, for an $M$ agent $K$ arm federated bandit problem, the CDP-MAB Algorithm\\
	$\bullet$ is $M\epsilon$-differentially private;\\
	$\bullet$ incurs communication cost $C_{C}(T)= O(c_1M\log T)$;\\
   $\bullet$ incurs regret $R_{C}(T)$ upper bounded by 
	\begin{equation}\label{regret1}
      	\textcolor{black}{  O(\max\{\sum_{k=1}^{K} \frac{\log(KT\log T)}{\Delta_k}, \frac{K\log T \sqrt{\log(KT\log T)}}{\sqrt{M}\epsilon}\}).}
	\end{equation}
\end{theorem} 
   \begin{remark}
 

Theorem 1 indicates a trade-off between the privacy and learning performance . The regret $O(\frac{K\log T \sqrt{\log(KT\log T)}}{\sqrt{M}\epsilon})$ increases with the decreasing of $\epsilon$ (higher level of privacy), until reaching $O(\frac{K\log(KT\log T)}{\Delta})$. The communication cost $C_C(T) = O(c_1M\log T)$ is the minimum cost needed to achieve $O(\log T)$ order regret. We next show that any $C(T)$ less than $c_1M\log T$ may bring the performance loss. 
 	
\end{remark}

 \textit{proof outline}: The privacy guarantee can be directly derived from the definition of the Laplace mechanism. Since the reward of each arm is bounded by $[0,1]$, the sensitivity (difference of the mean of each arm)  of two neighboring reward sequences is less than $\frac{1}{S(r) - S(r-1)}$.  Thus, the additional noise sampled from $Lap(\frac{1}{M\epsilon [S(r) - S(r-1)]})$ leads to $M\epsilon$-level privacy. The regret is incurred by playing suboptimal arms before they are correctly eliminated. We define $r_k$ as the epoch up to which $\Delta_k$ exceeds $2\widetilde{\Delta}_{r_k}$. Then we investigate the regret incurred by the following three events:  i) Exploration of each suboptimal arm $k$ before epoch $r_k$;  ii) Fail to eliminate suboptimal arm $k$  after the epoch $r_k$, and iii) Optimal arm is eliminated by server. Note that ii) and iii) can lead to at most $MT\Delta_k$ regret for arm $k$.  We show that after epoch $r_k =   \left\lceil\log (\frac{1}{\Delta_k})+1\right\rceil$, these events will not happen with a high probability. Yet regret caused by i) is unavoidable and dominates the total regret. Note that, all arms in the active arm set $I^{(r-1)}$ are pulled for the same number of times up to $S(r)$ in epoch $r$. Thus, the exploration number of arm $k$ can be bounded by $S(r_k)$. Calculating the total amount of exploration and multiplying by the reward gaps, we achieve the cumulative regret. Since  $r_k$ communication rounds are needed to identify the suboptimal arm $k$, we in total need $ \left\lceil\log (\frac{1}{\Delta})+1\right\rceil$ rounds to identify the best arm, leading to $c_1M\left\lceil\log (\frac{1}{\Delta})+1\right\rceil \leq O(c_1M\log T)$ communication cost. The full proof is provided in detail in Appendix.~B.

\subsection{Communication-Efficient CDP-MAB}
In real world applications, $c_1$ is usually large since it is energy-consuming for the devices to establish connections with a remote server, which leads to a large $C_c(T)$. We now consider how to extend the above Alg.~\ref{alg:cdp} to the case under communication constraints by utilizing efficient communication strategies as follows: 

$\bullet$ {\bf Less Frequent Communication.} A natural strategy is to communicate less frequently. Given the total time horizon $T$, we set $R$ as the constraint for the number of communication rounds. Therefore, $T$ can be divided into $R+1$ epochs. How to determine the length of each epoch $r\in R$ is not trivial. Given a fixed communication round $R$, if agents communicate too early, then their estimations of rewards are likely to be poor, as they are based on fewer samples. Thus, they may not be able to get the desired learning knowledge after all communication rounds. On the other hand, if they communicate too late, they cannot make full use of the information of others and the regret will scale with the number of agents. As a result, choosing the right time slots for communication is crucial.

The key idea to solve this is to ensure that the algorithm can identify the best arm after the $R$-th communication round with high probability. Otherwise, the time period from the end of $R$-th communication to $T$ may bring a linear growth of regret without further exploration. The doubling-increasing epoch length in Alg.~\ref{alg:cdp} may not meet this requirement since $R$ may be less than $\log(1/\Delta)$. To tackle this, we use an exponentially increasing length of epochs, scaled according to $R$: $\widetilde{\Delta}_r = \Delta^{\frac{r}{R}}$ (replacing Line 3 in Alg.~\ref{alg:cdp}).



$\bullet$ {\bf Partial Participation.} Another way to reduce the communication cost is to deactivate some of the communication links. We set $p$ as the link participation rate. We set $p$ as the communication link participation rate and $N= \left\lceil pM \right\rceil$ agents are enabled to communicate with the server during each communication round. 
In this way, $NS(r)$ samples can be observed at the server after the $r$-th communication round. Thus, the server needs to fine-tune the length of local exploration and elimination threshold according to $p$ in order to have a sufficient level of confidence to remove an empirically inferior arm. 

Specifically, we redesign $S(r)$ as $S^p(r) = \max\{\frac{8\log(8|I^{(r-1)}|r^2T)}{N\widetilde{\Delta}_r^2 }, \frac{8r \sqrt{2\log(8Kr^2T)} }{N^{1.5}\epsilon\widetilde{\Delta}_r}\}$ (replacing Line 4 in Alg.~\ref{alg:cdp}), and $C(r)$ as $C^p(r) =  \sqrt{\frac{\log (8|I^{(r-1)}|r^2T)}{2NS^p(r)}} +\frac{r\sqrt{8\log(8Kr^2T)}}{N^{1.5}\epsilon S^p(r)}$ (replacing Line 5 in Alg.~\ref{alg:cdp}). Compared with $S(r)$, $S^p(r)$ has a scale factor of $\frac{2^{-r}}{p\Delta^{r/R}}$ or $\frac{2^{-r}}{p^{3/2}\Delta^{r/R}}$ for the two terms inside the $\max\{\}$; so does $C^p(r)$ for its two terms compared with $C(r)$. At each communication round, the server randomly select $N$ users (i.e., $M-N$ sleep link) to upload their empirical arm means with additional noise sampled from $Lap(\frac{1}{N\epsilon [S^p(r)-S^p(r-1)})$. Then it reduces the active arm sized based on this partial aggregation as in Lines 11 to 14 in Alg.~\ref{alg:cdp}. We show the performance of the CDP-MAB algorithm under communication constraints in the following theorem.

\begin{theorem}[Performance of the CDP-MAB Algorithm Under Communication Constraints]
Given a participation rate $p$ and a limited number of communication rounds $ R$, the CDP-MAB algorithm under communication constraints,  \\
$\bullet$ is $M\epsilon$-differentially private;\\
$\bullet$ incurs communication cost $C^{p,R}_{C}(T) =   c_1\left\lceil pM \right\rceil R$;\\
$\bullet$ incurs regret $R^{p,R}_{C}(T)$ upper bounded by 
\begin{equation}\label{regret2}
      \textcolor{black}{ O(\min\{M,\frac{T^{2/R}}{p^{3/2}}\}\cdot \max \left\lbrace \sum_{k=1}^{K}\frac{\log (R^2KT)}{\Delta_k}\} ,\frac{ R\sqrt{\log (R^2KT)}}{\sqrt{M}\epsilon} \}\right\rbrace )}
\end{equation}	
	
\end{theorem}	
\begin{remark}
 $\frac{T^{2/R}}{p^{3/2}}$ can be seen as the performance loss term due to the limited communication. When $p = 1$ and $R$ is larger than $O(\log T)$, we can recover $R_{C}(T)$. If each agent runs CDP-MAB separately, we can obtain $MR_C(T)$ regret. In order to make our performance not worse than the non-communication case, we set the $\min\{M,\frac{T^{2/R}}{p^{3/2}}\}$ operation, which indicates that  $T^{2/R} /p^{3/2}\leq  M\rightarrow R \geq \frac{4\log T}{3 \log pM} $.
 \end{remark}
\textit{proof outline}: Similar to Theorem~\ref{thm:cdp}, we first prove that the event ${E_r}= \{|\bar{y}_k(r)-\mu_k| < {C^p(r)}\} $ occurs for each epoch $r$ with high probability according to the Hoeffding bound and the concentration of the Laplace distribution. We further argue that when $E_r$ holds for all epoch $r$, the best arm will not be eliminated from the active arm set and the suboptimal arm $k$ will be eliminated after epoch $r_k$ when $\Delta_k$ exceeds $2\Delta^{r_k/R}$. Thus, we need $\left \lceil  \frac{R(1+\log(\frac{1}{\Delta_k}))}{\log(\frac{1}{\Delta})}\right \rceil $ rounds to identify suboptimal arm $k$ and $R$ communication rounds to identify the best arm. Calculating the total amount of explorations before $r_k$ and multiplying by the reward gap $\Delta_k$, we achieve the cumulative regret. The full proof is in Appendix~C.


\subsection{Performance Analysis}

\subsubsection{Privacy-regret trade-off} 
The final order of $R_{C}(T)$ is determined by the relationship between $\sqrt{M}\epsilon/\sqrt{\log (T)}$ and the smallest suboptimal gap $\Delta$. If $\Delta<\frac{\epsilon \sqrt{M}}{\sqrt{\log (T)}}$, the first term dominate and we achieve $O(\frac{\log T}{\Delta})$ regret. Otherwise, we obtain a $O(\frac{\log ^{1.5} T}{\epsilon \sqrt{M}})$ regret. Indeed, this is determined by the two terms of $S(r)$. The first term in $S(r)$ can be considered as the number of samples needed to make $\bar{x}_k(r)$ concentrated with  $\mu_k$  within a certain confidence level. The second term can be treated as the number of samples to eliminate the effect caused by the added noise $\bar{y}_k(r)-\bar{x}_k(r)$. Clearly, if we require a stronger privacy level (smaller $\epsilon$), more noise needs to be added on $\hat{x}_{k}$. Then there is a larger difference between $\hat{x}_{k}$ and $\hat{y}_{k}$ and hence $S(r)$ is mainly determined by $\epsilon$. 

\subsubsection{Communication-regret trade-off} Here, $p$ and $R$ both indicate the trade-off between communication and learning performance. Specifically, they result in sub-linear and exponential deterioration terms. When $R$ is larger than $O(\log T)$, 
the terms $T^{2/R}$ and $T^{1/R}$ turn to be a constant and $p^{-{3/2}}{R_{C}(T)}$ regret can be achieved. When $p=1$, the worst-cast regret is still $MR_C(T)$ with $R = 3\log T/4(\log M)$. 
 

\section{ Decentralized Federated Multi-Armed Bandits}

In this section, we extend our CDP-MAB to the decentralized setting. We consider the communication network as an undirected graph $G(V,E)$ with vertices corresponding to the agents and directed edges depicting neighbor relationships. Without central coordination, the agents remove the inferior arms after aggregating information from their neighbors. However, the irregular connections lead to different local aggregation results after communication round $r$. Thus, each agent obtains the unequal-length active arm sets. This results in the \textit{asynchronous local exploration} phases among the agents at the start of epoch $r+1$. To avoid this, we need to ensure that: i) each agent has the same number of active arms at the beginning of each epoch. ii) the agents pull each active arm the same number of times during the local exploration phases. The key to achieve this is to make all agents reach the \textit{global consensus} through multiple information exchanges with neighbors. 

One of the simplest ideas is to use \textit{flooding} protocol. In \textit{flooding} protocol, an agent wishing to disseminate a piece of data across the network starts by sending a copy of this data to all of its neighbors. Whenever an agent receives new data, it makes copies of the data and sends the data to all of its neighbors, except the node from which it just received the data. The algorithm finishes when all the nodes in the network have received a copy of the data. Though \textit{flooding} can converge fast, it has implosion or the overlap problems. \textit{Gossiping} is an alternative to the flooding approach that uses randomization. Instead of indiscriminately forwarding data to all its neighbors, a gossiping agent only forwards data on to one randomly selected neighbor. However, gossiping distributes information slowly and incurs large end-to-end delay. In order to achieve fast convergence and avoid repeated message transmissions, we propose the following algorithm with a GIS (Global Information Synchronization) communication protocol.

\subsection{Algorithm Description}
DDP-MAB operates at each agent in epochs and each of them can be divided into following sub-phases:

\subsubsection{Local exploration}: Each agent $i$ perform $S(r)-S(r-1)$ times local exploration for each arm $k\in I^{(r-1)}$, update the empirical mean $\hat{x}_{i,k}(r)$ and transfer it to the private version. Noting, we chose the same $S(r)$ as we set in the CDP-MAB. 

\subsubsection{\textcolor{black}{GIS communication protocol}}

\textcolor{black}{The communication round starts after all agents finish their local exploration and ends when they receive private means from other $M-1$ agents. We assume there is an additional synchronization clock to inform the start and end of each communication round by monitoring the status of each agent $i$'s observation list $l_i^r(\cdot)$, which collects historical records of reward and arm selection information received from other agents.}

\textcolor{black}{Each communication round may contain \textit{multiple} time slots, each of them includes three handshake stages (ADV-REQ-DATA) for message exchange. For each agent $i$, the communication slot $n$ starts when it obtains new observations (private means) that it is willing to disseminate. It does this by sending a message ADV$_i(n)$ to its neighbors $j\in N_i$, naming the agent labels (ADV stage). Upon receiving an ADV$_i(n)$, the neighboring node $j$ checks to see whether it has already received the advertised observations. If not, it responds by sending an REQ$_j(n)$ message for the missing observations back to the sender $i$ (REQ stage). The communication step completes when $i$ responds to the REQ$_j(n)$ with a DATA$_{i,j}(n)$ message, containing the missing observations (DATA stage).  Fig.~2 shows an example of the protocol using two-step communication and the details of the GIS protocol is shown in Alg.~2.}

\begin{algorithm}[H]
	\caption{\textcolor{black}{GIS Communication Protocol} } 
	{\bf Initialization:} 
	counter $n = 0$;
	An observation list $l_i^r(0) = \{i\}$ for all $i$.
	\begin{algorithmic}[1]
\While{Communication round not ends}
    \For{each agent $i$}
	 \State \textcolor{black}{\textit{Local exploitation}: Pull the local empirical best arm once};
    \If {New observations $\bar{\mathbf{y}}_f(r)$ are added to $l_i^r(n)$}
    \State Send message $ADV_i(n) = \{f\}$ to neighbors $j\in N_i$
    \State Receive the $REQ_j(n)$ messages from $j\in N_i$. 
    \State Send $DATA_{ij}(n)$ back to serve the requests.
    \EndIf
    \If {Receive $ADV_j(n) = \{g\}$ message from $j\in N_i$ }
    \State Check whether $\bar{\bf{y}}_g(r)$ exist in $l_i^r(n)$. If not, send $REQ_i(n) =\{g\}$ back to $j$;
    \State Receive the information from $DATA_{ji}(n)$ and update the $l_i^r(n)$.
    \EndIf
    \EndFor
    \If {$|l_i^r(n)| = M$ for all $i$ }
    \State Communication round ends and return $t_{delay} = n$.
    \Else
    \State $n = n+1$.
    \EndIf
\EndWhile
	\end{algorithmic}
\end{algorithm}

\subsubsection{\textcolor{black}{Local exploitation}} \textcolor{black}{Each communication round lasts $t_{\text{delay}}$ time slots, which can be seen as the delay suffered before information synchronization. Instead of waiting in the communication round, all agents are required to pull the best empirical arm based on local observation for $t_{\text{delay}}$ times}. 

\subsubsection{Local aggregation and elimination} \textcolor{black}{When communication ends, each agent $i$ generates the aggregated mean $\bar{y}_{k}(r)= \frac{1}{M}\sum_{j\in M}\bar{y}_{j,k}(r)$ for each arm $k$. After that, each agent $i$ performs local elimination with $C(r)$ (Lines 9 and 10 in Alg.~3)}. 
\begin{figure}[H]
	\centering
	\includegraphics[width=1\linewidth]{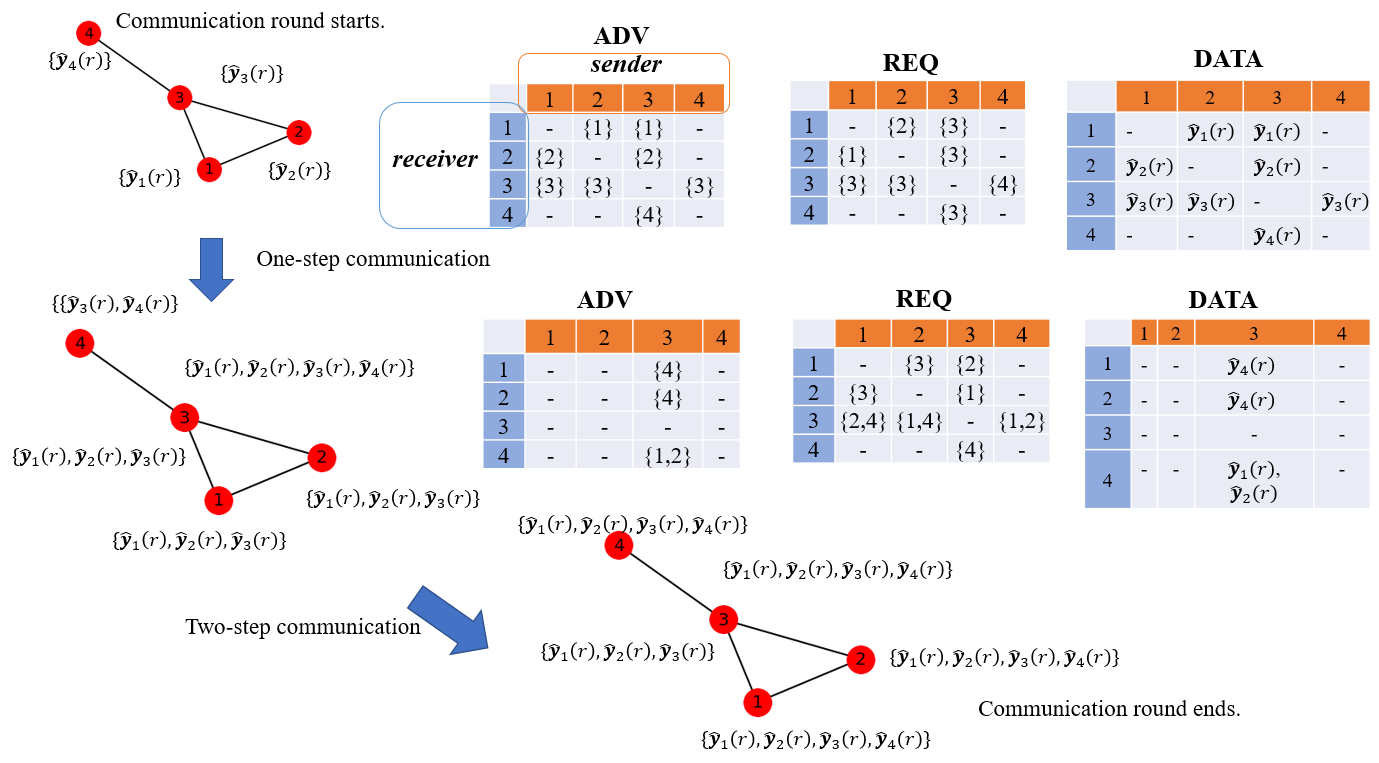}
	\caption{An example of GIS communication protocol.}
	\label{hybrid}
\end{figure}
 \begin{algorithm}[H]
	\caption{\textcolor{black}{Decentralized Differentially Private Multi-armed Bandit Algorithm (DDP-MAB)}} 
	
	{\bf Input:}
	Time horizon: $T$;  Privacy parameter $\epsilon$; number of agents $M$; communication graph $G(V,E)$;\\ 
	{\bf Initialization:} 
    $t = 1$, $r=1$; $I^{(0)} = [K]$; $x_{i,k}(1) = \hat{y}_{i,k}(1) = \bar{y}_{i,k}(1) = 0$; $S(0) = 0$
	\begin{algorithmic}[1]
	\While{$t<T$}
		\While {$|I^{(r-1)}|>1$} 
		\State $\widetilde{\Delta} _r\leftarrow 2^{-r}$ 
		\State Set $S(r) \leftarrow
 		 \max\{\frac{8\log(8|I^{(r-1)}|r^2T)}{M\widetilde{\Delta}_r^2 }, \frac{8r \sqrt{2\log(8Kr^2T)} }{M^{1.5}\epsilon\widetilde{\Delta}_r}\}$
 		 \State Set $C(r) \leftarrow  \sqrt{\frac{\log (8|I^{(r-1)}|r^2T)}{2MS(r)}} +\frac{r\sqrt{8\log(8Kr^2T)}}{M^{1.5}\epsilon S(r)}$
		\For{each agent $i = 1,...,M$} 
		\State \textit{Local exploration}: Run Line 7 - 10 in Alg.1;
		\State \textcolor{black}{ \textit{Communication}: Run \textit{GIS Communication Protocol}};
		\State \textit{Aggregation} at agent $i$: $\bar{y}_{k}(r)=\frac{1}{M}\sum_{j\in M}\hat{y}_{j,k}(r)$
		\State \textit{Elimination }at agent $i$: Remove  $k$ from $I^{(r-1)}$ if $\bar{y}_{max}(r) -\bar{y}_{k}((r) )\geq 2C(r)$	
	    \EndFor
		\State $t =I^{(r-1)}(S(r) - S(r-1)) +t_{delay}$, $r = r+1$
		\EndWhile
		\State All agents pull the arm until time $T$.	
		\EndWhile
	\end{algorithmic}
\end{algorithm}

 \begin{theorem}[Performance of DDP-MAB]
 	Given time horizon $T$ and privacy parameter $\epsilon$, the cost for building an agent-to-agent link $c_2$, for the $M$ agents equipped with $K$ arms communication over the graph $G$, Algorithm 3\\
 	$\bullet$ is $(M\epsilon)-$ differentially private;\\
 	$\bullet$ \textcolor{black}{incurs communication cost $C_{D}(T) = O( c_2d_G\log T \sum_{i=1}^M d_i/2)$};\\
 	$\bullet$ \textcolor{black}{incurs regret $R_{D}(T)$ upper bounded by}
\begin{equation}\label{regret5}
    \textcolor{black}{O(\max\{\sum_{k=1}^{K} \frac{\log(KT\log T)}{\Delta_k}, \frac{K\log T \sqrt{\log(KT\log T)}}{\sqrt{M}\epsilon}\}+ M(d_G-1))}
 	\end{equation}
 \end{theorem}
where $d_G$ is the diameter of $G$ and $d_i$ is the degree of agent $i$.

\textit{proof outline}: 
The communication cost is determined by the number of agent-to-agent links established in all communication rounds. Since i)All agents can synchronously identify the best arm after $\left \lceil \log (\frac{1}{\Delta_k}) +1 \right \rceil $ rounds; ii) Each communication rounds last $t_{\text{delay}}$ time slots, upper bounded by the diameter of the graph $d_G$; iii) In each time slot, at most  $\sum_{i=1}^M d_i/2$ connections are established. Multiply these three items we can conclude $C_{D}(T)$.

The first term in regret is caused by local explorations of all $M$ agents before they eliminate all suboptimal arms, which recover $R_{C}(T)$ by the same choice of $S(r), C(r)$. The second term is incurred by local exploitation. The communication round $r$ contains at most $d_G-1$ time slots. If all $2^{-(r-1)} $ suboptimal arms are successfully eliminated in the previous round, and arm 1 is not eliminated, at most $M(d_G-1)2^{-(r-1)}$ regret will be introduced. If arm 1 is eliminated, then at most $Md_G$ regret will be introduced.  We complete the proof by summing up the regrets incurred by all required communication rounds. The detailed proof can be found in Appendix. D.

To constrain the total communication round as $R$ to reduce the communication cost, we can just set $\widetilde{\Delta} _r\leftarrow \Delta^{r/R}$ in Line 3 of Alg.3, which leads to the following results: 
\begin{corollary}
DDP-MAB with communication round constrain $ R $ achieves the regret of  $$ R^R_D(T) = O(\min\{M,{T^{2/R}}\}\cdot \max \left\lbrace  \sum_{k=1}^{K}\frac{\log (R^2KT)}{\Delta_k} ,\frac{ R\sqrt{\log (R^2KT)}}{\sqrt{M}\epsilon} \right\rbrace   +  M(d_G-1))$$
with communication cost $C^R_D(T) = c_2d_GR\sum_{i\in[M]}d_i/2$.
\end{corollary}
The detailed proof is shown in Appendix. E.
\begin{remark}
 When $R$ is larger than $O(\log T)$, we can recover $R_{D}(T)$. To ensure our performance not worse than the non-communication case, $R \geq 2\log T /\log M$. 
\end{remark}
\subsection{Performance Analysis}

\subsubsection{Communication cost}
\textcolor{black}{ In the centralized setting, each communication round only accounts for one time slot. Thus the communication cost is determined by $c_1$, the participating agents and the number of total communication rounds. In the decentralized setting, the communication cost is jointly decided by $c_2$, the total communication round and the links built during each round. If $c_1M = c_2d_G\sum_{i\in[M]}d_i/2$, the communication costs incurred by these two settings are the same.}


\subsubsection{Regret}
\textcolor{black}{ Compare with $R_{C}(T)$, there is an additional term $O(Md_G)$ in $R_{D}(T)$, which can be regarded as the extra regret incurred by the inconsistency between local estimation and global estimation. $d_G$ indicates the convergence rate of local estimation to global estimation on graph $G$. This term only depends on the agent number $M$ as well as the diameter of the graph $d_G$, and not depending on $T$.} 

\subsubsection{Trade-off}
\textcolor{black}{ There is a non-proportional trade-off relationship between the above two. Sparse graphs build fewer agent-to-agent links in each communication round, but suffer a considerable delay before information synchronization; Dense graphs have a fast convergence rate but demand more connections. The graphs having both fewer edges and smaller $d_G$, such as star or multi-star graphs that are close to the centralized setting, can achieve the best performance. We summarize the communication costs and regrets of several typical graphs in Table. I}.
\begin{table}[h]
\caption{Communication cost and Regret of several typical graphs}
\small
\centering
\begin{tabular}{|c|c|c|c|c|}
\hline
& Star      & Ring       & Fully-connected               & d-regular               \\ \hline
$C_{D}(T)$ & $O(2(M-1)\log T$ & $O((M^2/2)\log T)$ & $O(\frac{M((M-1)}{2} \log T)$ & $O(\frac{dM\log M}{2\log (1/\lambda_2)}\log T$\footnotemark[1]      \\ \hline
$R_{D}(T)$ & $R_C(T)+O(M)$    & $R_C(T)+O(M(M-1))$ & $R_C(T)$                  & $ R_C(T)+O(\frac{M\log M}{\log (1/\lambda_2)})$ \\ \hline
\end{tabular}
\end{table}

\footnotetext[1]{$\lambda_2$ is the second largest eigenvalue of the Laplace matrix of the graph. We omit $c_2$ in $C_D(T)$ for all graphs.}

\subsubsection{Discussion}
\textcolor{black}{$R_{D}(T)$ achieves the similar form of regret as in \cite{zhu2021federated}. That is, the decentralized regret is equal to the centralized regret plus another graph-related term. In our work, this term is determined by $M$ and the diameter of the graph $d_G$, while in \cite{zhu2021federated}, this term is determined by $M$ and the eigenvalue of the graph Laplacian matrix. In fact, both of these two terms indicate the convergence rate of local estimation to global estimation on specific graph structures. 
Although we can finally achieve a consistent regret using different communication protocols, their method requires $O(T)$ communication cost which is not communication-efficient. Besides, the privacy mechanism they used results in a $O(K\log^{2.5}(T)/\epsilon)$ regret. Compared to them, we only need $O(\log T)$ communication cost and the privacy-related regret is upper bounded by $O(K\log^{1.5}(T)/(\epsilon\sqrt{M}))$. This also demonstrates the trade-off between communication and privacy: more frequent communications require more noise to be added, which further leads to a larger regret.}

\section{Hybrid Differentially Private Multi-Armed Bandit Algorithm }

In practice, CDP-MAB builds $M$ server-to-agents links, thereby introducing high communication cost. Although we propose a partially sampling method to achieve communication efficiency in Section IV-B, it inevitably brings performance loss. DDP-MAB is more suitable for devices in a small-size network, otherwise the delay for reaching consensus in each communication round is unacceptable. In this section, we propose a hybrid communication structure (see Fig.~3) that combines the centralized and decentralized settings. This structure is a natural extension of the classic wireless sensor network (WSN). Each sub-network in WSN contains some general sensors and a sink node, that can communicate with the WSN server through the gateway after collecting information in the sub-network. 

There are $Q<<M$ components, each of them consists $M_q$ agents. The agents belonging to the same component are allowed to communicate over a sub-graph $G_q(E_q,V_q)$. There also exists a sever, coordinating the communication among components. Assume there is a sink agent $SA_q$ of component $q$. The server only communicates with the sink agents, which scales down the number of server-to-agent links from $M$ to $Q$ .

\begin{figure}[h]
	\centering
	\includegraphics[width=0.8\linewidth]{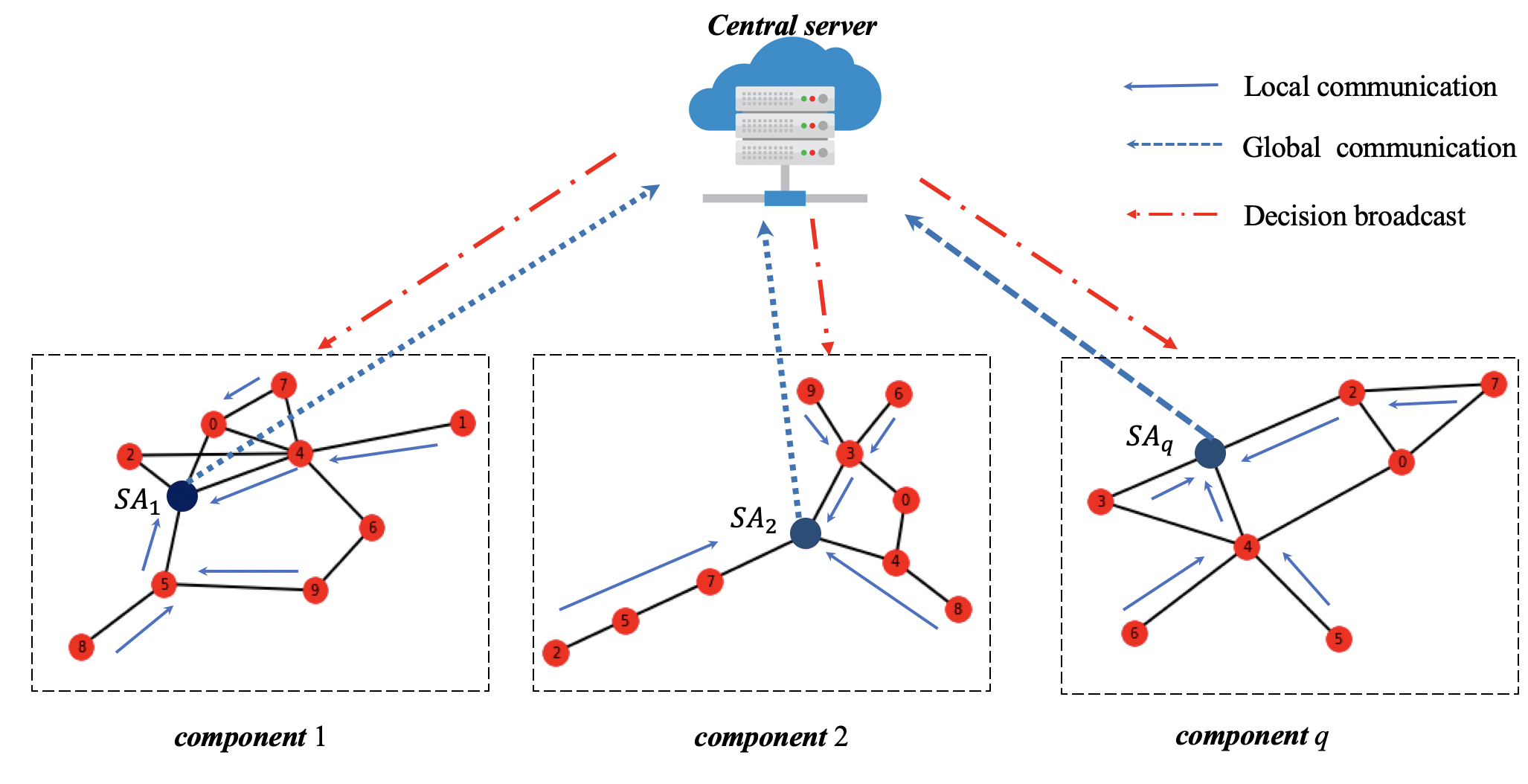}
	\caption{Hybrid communication protocol for federated bandits.}
	\label{hybrid}
\end{figure}

\subsection{Algorithm Description}
HDP-MAB operates in epochs and each of them can be divided into the following sub-phases:
\subsubsection{Local exploration} Each agent $i$ perform up to $S(r)$ times local exploration for each arm $k\in I^{(r-1)}$, update the empirical mean $\hat{x}_{i,k}(r)$ and transfer it to the private version. Noting, we chose the $S(r)$ as we set in the CDP-MAB and DDP-MAB. The communication round starts when agents finish their local exploration and ends when all agents receive the update active arm set from the server. We utilize a two-layer communication protocol:
\subsubsection{Local communication and aggregation} \textcolor{black}{Different from the fully-decentralized setting, we no longer use the GIS protocol to achieve ``global information synchronization''. Instead, we use the Sink Agent Collection (SAC) protocol to realize ``one-way message passing'', where all agents in $q$ send their private means to $SA_q$. The local communication in $q$ ends when $SA_q$ observes information from other $M_q-1$ agents. This inevitably introduces some local communication delay, defined as $t_{\text{delay}}^q$. However, we can wisely select $SA_q$ to minimize this delay by Alg.~4. Given the communication graph $G_q$ of component $q$, we randomly assign any agent $i$ as the sink agent, and calculate the shortest path from all other agents to $i$. The maximum value among them is the local delay introduced by $SA_q = i$, denoted as delay$_i$. Then, $SA_q$ should be the agent with minimal local delay: $SA_q = \arg \min_{i\in q}\{\text{delay}_i\}$}.  
After local communication, the sink agents perform local aggregations.
\begin{algorithm}[H]
	\caption{Find Sink Agents} 
	
	{\bf Input:} 
  set of communication graph $\{G_q(V_q,E_q)\}_{q=1}^Q$\\
   {\bf Output:}  
	Sink Agents $\{SA_1,...,SA_Q\}, t_{\text{delay}}$
	\begin{algorithmic}[1]
	\For{component $q = 1$ to $Q$}
	\For {each agent $i\in[M_q]$ }
    \State compute the shortest distances $sd(i,j)$ between agent $j\neq i$ and $i$;
    \State delay$_i = \max\{sd(i,j)\}-1$;
    \EndFor
    \State $SA_q = \arg \min_{i\in q}\{\text{delay}_i\}$, $t_{\text{delay}}^q = \max_{i\in [M_q]}\{sd(i,SA_q)\} $
    \EndFor
    \State $t_{\text{delay}} = \max_{q\in Q}\{t_{\text{delay}}^q \}$
    
	\end{algorithmic}
\end{algorithm}
\begin{remark}
 In the fully-decentralized setting, each node cannot access the structure of the entire graph, but can only gradually exchange messages with its neighbors. While in the hybrid structure, the server has stronger information collection and computing capabilities, so that it can find all sink agents $\{SA_1,...,SA_Q\}$ by running Alg.~4. 
\end{remark}

\subsubsection{Global communication and aggregation} The global communication will be exactly the same as the protocol in Alg.~1 by treating the sink agents as participants. After the server collects the information, it performs global aggregation and elimination, and finally broadcasts the updated results to all agents.

\subsubsection{Local exploit} \textcolor{black}{Since global communication does not introduce extra delay, the delay of the whole communication round is determined by the slowest component that completes the sink agent collection, which is $t_{\text{delay}} = \max_{q\in Q}\{t_{\text{delay}}^q\} $. All agents exploit the locally observed best empirical arm during $t_{\text{delay}}$}.

 \begin{algorithm}[H]
	\caption{\textcolor{black}{Hybrid Differentially Private Multi-armed Bandit Algorithm (HDP-MAB)} }
	
	{\bf Input:}
	Time horizon: $T$;  Privacy parameter $\epsilon$; number of agents $M$; a set of communication graph $\{G_q(V_q,E_q)\}_{q=1}^Q$\\
	{\bf Initialization:} 
	 $t = 1$, $r=1$; $I^{(0)} = [K]$; $x_{i,k}(1) = \hat{y}_{i,k}(1) = \bar{y}_{i,k}(1) = 0$; $S(0) = 0$;  \\
	$\{SA_1,...,SA_Q\} = \text{FindSinkAgents}(G_1(V_1,E_1),...,G_Q(V_Q,E_Q))$
	\begin{algorithmic}[1]
	\While{$t<T$}
		\While {$|I^{(r-1)}|>1$}
		\State Set $\widetilde{\Delta} _r\leftarrow  2^{-r}$
		\State Set $S(r) \leftarrow
 		 \max\{\frac{8\log(8|I^{(r-1)}|r^2T)}{M\widetilde{\Delta}_r^2 }, \frac{8r \sqrt{2\log(8Kr^2T)} }{M^{1.5}\epsilon\widetilde{\Delta}_r}\}$
 		 \State Set $C(r) \leftarrow  \sqrt{\frac{\log (8|I^{(r-1)}|r^2T)}{2MS(r)}} +\frac{r\sqrt{8\log(8Kr^2T)}}{M^{1.5}\epsilon S(r)}$
		\For {component $q = 1$ to $Q$ }
		\State \textit{Local exploration}: Agent $i \in [M_q] $ runs Line 7 - 10 in Alg.1;
	    \State \textit{Local communication}: Agent $i \in [M_q] $ sends $\bar{y}_{i,k}(r)$ to $SA_q$; 
	     \State \textit{Local aggregation}: $SA_q$ aggregates $\bar{y}_{q,k}(r)=\frac{\sum_{j\in q} \bar{y}_{j,k}(r)}{M_q}$; 
		\State \textit{Global communication}: $SA_q$ upload $\bar{y}_{q,k}(r)$ to the server;
         \State \textit{Local exploitation}: Agent $i \in [M_q]$ keeps pulling local empirically best arm until receive the update $I^{(r)}$ from server. 
		\EndFor
	\For{the server}
	\State Receive messages from all sink agents.
	\State \textit{Global aggregation}:
	$\bar{y}_{k}(r)=\frac{\sum_{q\in Q}\bar{y}_{q,k}(r)}{Q}$
    \State \textit{Global elimination}: remove $k$ from $I^{(r-1)}$ if $\bar{y}_{max}(r) -\bar{y}_{k}(r)\geq 2C(r)$
    \State Broadcast the update $I^{(r)}$ to all agents.
    \EndFor
    \State $t =  |I^{(r-1)}|(S(r) - S(r-1))+ t_{\text{delay}}$, $r = r+1$, 
	\EndWhile
	\State All agents pull the arm until time $T$.	
	\EndWhile
\end{algorithmic}
\end{algorithm}

\begin{theorem}[Performance of HDP-MAB]
		Given time horizon $T$, privacy parameter $\epsilon$ and the communication cost weight $c_1$, $c_2$. Consider $Q$ components are connected to a server. Inside component $q$, $M_q$ agents with $K$ arms communicate over a graph $G_q$. Algorithm 5\\
$\bullet$ is $(M\epsilon)-$ differentially private;\\
 \textcolor{black}{	$\bullet$ incurs communication cost $C_{H}(T) = O( (c_2M\max_{q\in Q}\{d_G^q\} +c_1Q) \log T)$};\\
  \textcolor{black}{	$\bullet$ incurs regret $R_{H}(T)$ upper bounded by} 
\begin{equation}\label{regret5}
     \textcolor{black}{ O(\max\{\sum_{k=1}^{K} \frac{\log(KT\log T)}{\Delta_k}, \frac{K\log T \sqrt{\log(KT\log T)}}{\sqrt{M}\epsilon}\})+ M\max_{q\in Q}\{d_G^q-1\}}
 	\end{equation}
\end{theorem}
where $d_G^q$ is the diameter of component $q$. The detailed proof can be found in Appendix.F.

To constrain the total communication rounds as $R$ to realize communication efficiency, we can just set $\widetilde{\Delta} _r\leftarrow \Delta^{r/R}$ in Line 3 of Alg.5, which leads to the following results: 
\begin{corollary}
 Alg.5 with communication round constrain $ R $ achieves the regret of 
 \begin{equation}
   \textcolor{black}{R^R_H(T) = O(\min\{M,{T^{2/R}}\}\cdot\max \left\lbrace  \sum_{k=1}^{K}\frac{\log (R^2KT)}{\Delta_k} ,\frac{ R\sqrt{\log (R^2KT)}}{\sqrt{M}\epsilon}\right\rbrace  + M\max_{q\in Q}\{d_G^q\})}  
 \end{equation} with communication cost $C^R_H(T) = O(c_2M\max_{q\in Q}\{d_G^q-1\} +c_1Q) R$.
\end{corollary}
This result can be obtained by combining the proof of Theorem 4 and Corollary 1.
\begin{remark}
 When $R$ is larger than $O(\log T)$, we can recover $R_{H}(T)$. To ensure our performance not worse than the non-communication case, $R \geq 2\log T /\log M$.
 \end{remark}

\subsection{Performance Analysis}

\subsubsection{Communication cost}
\textcolor{black}{Employing the hybrid structure is helpful to achieve communication efficiency. On the one hand, the GIS protocol only allows information exchange between neighbors to completely diffuse the information. The communication cost thereby is determined by $d_G$ and the number of graph edges. The SAC protocol realizes one-way information aggregation to a fixed sink node with the help of the server, which reduces the number of agent-to-agent links inside a component. On the other hand, the number of server-to-agents links is only proportional to the number of components $Q$, not $M$, which significantly reduces the burden of the uplink}.


\subsubsection{Regret}
\textcolor{black}{Compared with $R_{D}(T)$, the additional term of $R_C(T)$ no longer depends on $d_G$ but $\max_{q\in Q}\{d_G^q\}$. If $M$ is distributed in different components, the size of each sub-graph and corresponding local delay decrease, which ultimately leads to the reduction of the regret introduced by the local exploitation. In one extreme case, when $Q=M$, each agent is directly connected to the server, then the delay of each communication round is 0 and we recover $R_{C}(T)$. In addition, we can recover $R_{D}(T)$ with $Q=1$}. 

\subsubsection{Trade-off} \textcolor{black}{The above results provide important insights into designing practical communication efficient federated MAB systems. i) Instead of utilizing the master-worker structure, we can reduce the number of server-to-agent links by disturbing the agents into different components. ii) Second, the unbalanced agent distribution may cause a large delay as it is determined by the slowest component that completes the local communication. Thus, we should try to ensure the balance of agents in each component. iii) Finally, enforcing the sub-graph to be close to the center (or multi-center) rather than fully connected or ring structure can also help reduce communication costs as well as regrets.}

\subsection{Discussion}

\subsubsection{Comparison of three communication structures}
     \textcolor{black}{We summary the procedure of one specific epoch for three algorithms in Fig. \ref{compare} and compare their performance in terms of communication cost, delay and regret.}

    \begin{figure}[H]
	\centering
	\includegraphics[width=1\linewidth]{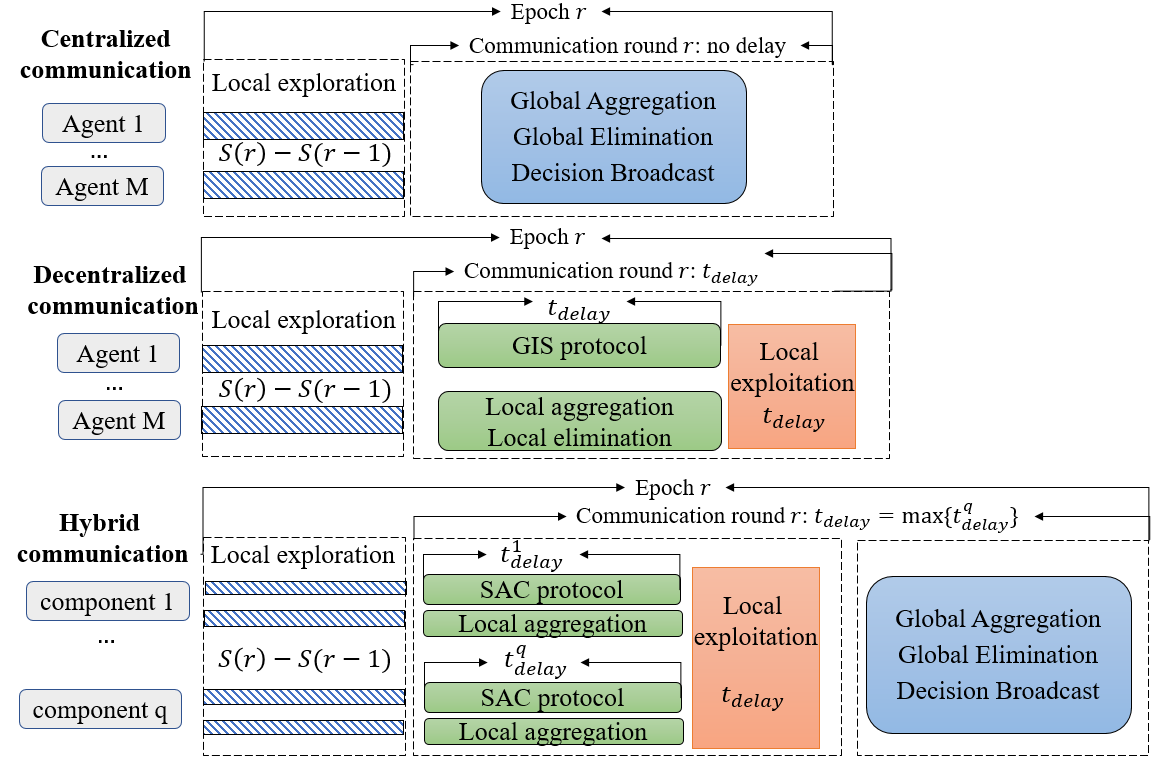}
	\caption{Comparison of three communication structures.}
	\label{compare}
\end{figure}

\begin{itemize}
\item \textbf{Communication cost}.  \textcolor{black}{$C_C(T)$ is determined by the number of  server-to-agent links and the number of communication rounds required to remove all inferior arms; $C_D(T)$ and $C_H(T)$ are decided by he number of built agent-to-agent links, the required communication rounds as well as the $t_{\text{delay}}$ of each round. }
\item \textbf{Delay}.  \textcolor{black}{The centralized setting does not introduce delay. The $t_{\text{delay}}$ in decentralized communication is introduced by the GIS protocol, and in hybrid communication it is the maximum value of delay introduced by local communication among all components.}
\item \textbf{Regret}. \textcolor{black}{$R_{C}(T)$ achieves $O(\log T)$ regret. Both $R_D(T)$ and $R_H(T)$ have the addable items based on $R_{C}(T)$, which are caused by the local exploit during $t_{\text{delay}}$. The items are only related to the graph structure and do not scale with $T$. Thus, both three algorithms can achieve the same $O(\log T)$ order regret without considering privacy.} 
\end{itemize}

\subsubsection {Heterogeneous reward structure}
\textcolor{black}{We now discuss the extension of the above algorithms to the heterogeneous rewards setting. Similar to \cite{Shi_Shen_2021,zhu2021federated}: We consider 
that for arm $k$ and players $i,j$, the expected mean $\mu_{i,k}$ and $\mu_{j,k}$ are not equal in general. There exists a  \textit{true} reward or \textit{global} reward of arm $k\in [K]$ that equals to the average of the means of all agents’ expected rewards:
$\mu_k = \frac{1}{M}\sum_{i=1}^M \mu_{i,k}$, that implies the true reward can be obtained by averaging and thus cancelling out local biases. No individual agents can make a correct inference by simply collecting individual rewards. Therefore, they must collaborate with each other to estimate the true rewards in a federated fashion.}     
\textcolor{black}{ We claim that our methods in this work can be applied to the above heterogeneous setting without modification.}

\begin{corollary}
 Under the heterogeneous reward setting, 
 \begin{itemize}
\item      \textcolor{black}{Our CDP-MAB can achieve $R_C(T)$ regret} ;
\item       \textcolor{black}{Our DDP-MAB can achieve $R_D(T) + O(M(d_G-1)\log (T))$ regret};
\item       \textcolor{black}{Our HDP-MAB can achieve $R_H(T) + O(M\max_{q\in Q}\{d_G^q-1\} \log T)$ regret.}
     
 \end{itemize}
\end{corollary}
\begin{proof}
\textcolor{black}{The core idea of our work and [2][21] is the same: Through communication, the estimated mean at the server (or each agent) can converge on the average value of empirical means from all $M$ agents, which can cancel local bias. In the centralized and hybrid setting, the server can capture the global knowledge without local bias. In other words, the server can observe unbiased estimation of true means, which directly avoids the impact of heterogeneity. In the decentralized setting, our proposed GIS protocol ensures that each agent can receive empirical means from other $M-1$ agents in each epoch. Then, the local estimated mean (at each agent) can converge to the averaged empirical means in finite time slots. That is, each agent can reach consensus, like the server in the centralized setting. Note that both our decentralized and hybrid algorithm have a \textit{local exploitation} phase. During this period, we need to pull the local empirical optimal arm, which may be inconsistent with the global optimal one. Since we need $O(\log T)$ communication rounds, each of them lasts for at most $d_G$ (in the decentralized setting) or $\max_{q=1}^Q\{d_G^q\}$ (in the hybrid setting) time slots. So local exploitation can introduce at most $O(\log T)$ order regret under the heterogeneous reward setting, which does not affect the final regret order.}
\end{proof}

\section{Experiments}

In this section, we conduct experiments to empirically verify the theoretical results of previous sections, that is, the trade-offs between communication, privacy and learning regret under different communication protocols. 

\subsection{Centralized Setting}
\textit{Experimental Settings}: We consider $M=50$ agents connected with a central server. Each of them plays a Bernoulli MAB with 100 arms. The means are randomly generated from $[0,1]$. We set $c_1$, the communication cost to build a server-to-agent link, equal to 25. We first consider the homologous reward structure, where all agents see the same set of arm means.  
\subsubsection{Privacy-regret trade-off}
In this part, we allow all agents to communicate with the server at each communication round and only investigate the effect of privacy level $\epsilon$. We consider 4 different privacy parameters $\epsilon = \{0.1,0.3,0.5,1\}$. Fig. \ref{fig:epsilon} shows that the regret increases with the decreasing of  $\epsilon$ (higher level of privacy), since larger noises are added on the local update estimations. $\epsilon=1$ indicates the non-private case.
\begin{figure}[H]
	\centering
	\includegraphics[width=0.6\linewidth]{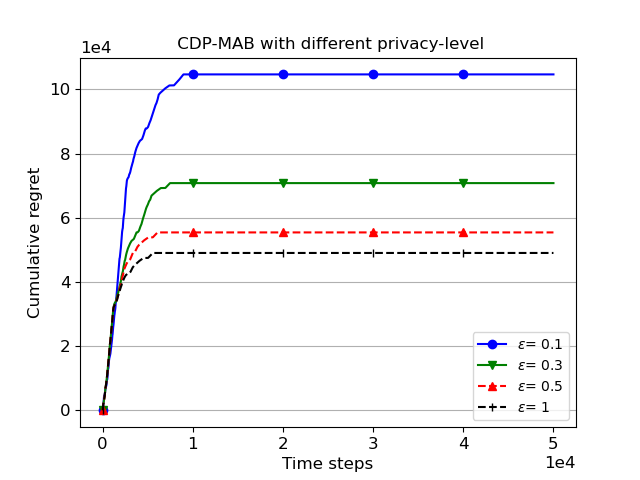}
	\caption{Regrets of CDP-MAB with different privacy level $\epsilon$.}
	\label{fig:epsilon}
\end{figure}

\subsubsection{Communication-regret trade-off}
We then fix the privacy-preserving level and discuss the effect of communication-reduction strategies.  

\textit{Participating rate $p$}: we first consider 5 different participating rates $p = \{0.2,0.4,0.6,0.8,1\}$. Among that, $p=1$ indicates the fully participating case that all agents send their perturbed models to the server after each epoch. Figure. \ref{fig:probfull} shows that the regrets at $T$ decrease as $p$ increases.


\textit{Communication round $R$}: We then fix the participating rate as $p=1$, and show the effect of the communication rounds constraints. We consider the number of communication rounds $R$ varying in $\{2,3,4,5\}$. Fig.\ref {fig:rfull} shows that the regrets decrease with $R$ increase. It is worth noting that the regret we reached at $ R=4$ and 5 are almost the same. This is because we need about $ \log(T) \approx 4$ communication rounds before eliminating all suboptimal arms. When $R>4$, the term $T^{-2/R}$ in Thereon 2 tends to a constant and do not affect the total regrets.  
\begin{figure}[H]
	\centering
	\begin{minipage}[t]{0.45\textwidth}
		\centering
		\includegraphics[width=7.5cm]{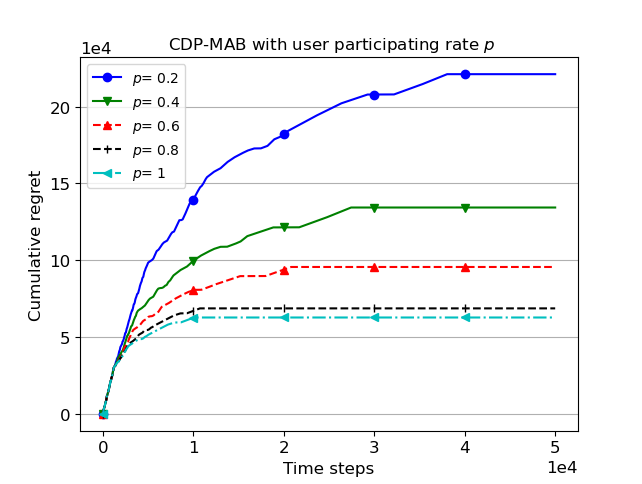}
		\caption{Regrets of CDP-MAB with participating rate $p$.}
		\label{fig:probfull}
	\end{minipage}
	\begin{minipage}[t]{0.45\textwidth}
		\centering
		\includegraphics[width=7.5cm]{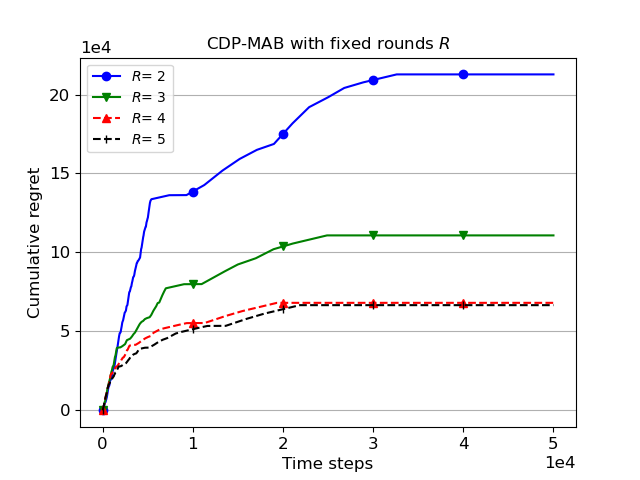}
		\caption{Regrets of CDP-MAB with rounds $R$.  }
			\label{fig:rfull}
	\end{minipage}
\end{figure}

\textit{Combination of $p$ and $R$}: we finally combine these two communication constraints $p$ and $R$ with privacy parameter $\epsilon = 1$. With different pairs of $(p,R)$, we compare the final cumulative regret $R_{C}(T)$ achieved at time slot $T$ as well as the total communication cost $C_{C}(T)$. Fig.~\ref{fig:combpr} shows a clear trade-off between the communication cost and learning performance. Note that $R=5,p=1$ is close to the largest amount of required communication cost of the centralized setting. Correspondingly, the regret of this case is close to the performance in Fig. \ref{fig:epsilon} with $\epsilon = 1$.

\begin{figure}[htbp]
	\centering
	\begin{minipage}[t]{0.48\textwidth}
		\centering		\includegraphics[width=7cm]{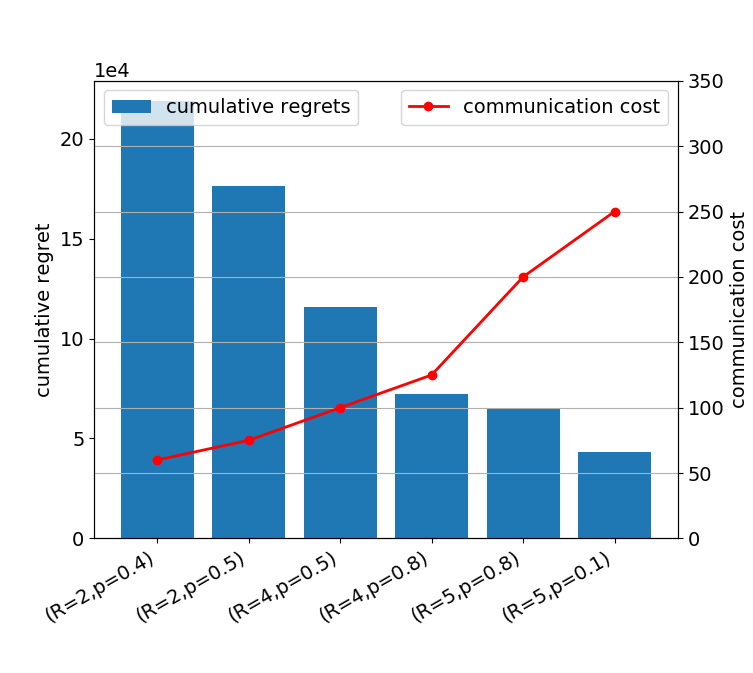}
		\caption{Regrets of CDP-MAB with participate rate $p$ and communication rounds $R$.}
		\label{fig:combpr}
	\end{minipage}
	\begin{minipage}[t]{0.48\textwidth}
		\centering
			\includegraphics[width=8cm]{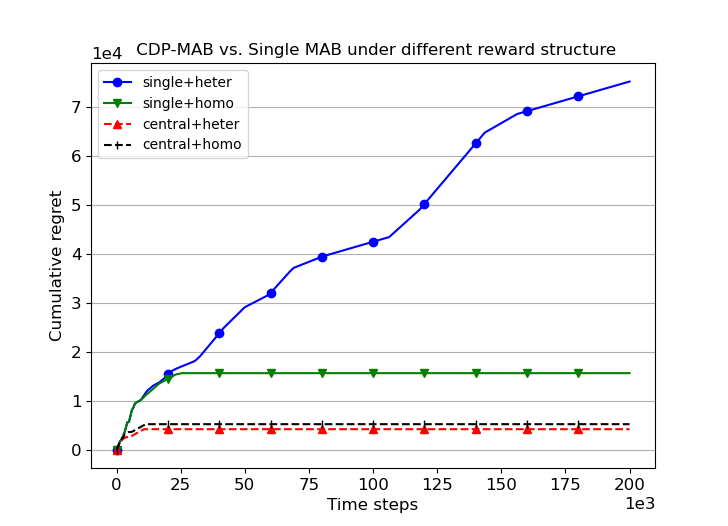}
	\caption{Regrets of Single-MAB vs. CDP-MAB under different reward structures.}
	\label{fig:central_h}
	\end{minipage}
\end{figure}
\subsubsection{Heterogeneous reward setting}
\textcolor{black}{We use a 5-agent 10-armed small instance to illustrate the effectiveness of CDP-MAB on heterogeneous rewards. For the definition of heterogeneous reward, please refer to Section VI-D. From Fig.~\ref{fig:central_h} we see that our proposed method obtain similar regrets under two reward structures, which  reached $1/M$ of the single-agent performance under homogeneous reward. For heterogeneous rewards, the single-agent method cannot converge due to the inconsistency between the local and global best arm. After completing the local best arm identification and entering the exploit phase, the inconsistency will bring a linear regret with $T$}.

\subsection{Decentralized Setting}
\textcolor{black}{In this part we investigate how the properties of graphs can affect the learning performance in the decentralized setting.} 
\subsubsection{Network structure}
\textcolor{black}{We consider a 50-agent 100-armed homogeneous problem instance. We set $c_2$, the communication cost for each agent-to-agent link as 1. The agents are connected based on four kinds of graphs: \{Fully-connected, Star, Ring, Random\}. Some examples are shown in Fig.~\ref{fig:graph_example}}. 

\begin{figure}[htbp]
	\begin{minipage}[t]{0.45\textwidth}
		\includegraphics[width=7cm]{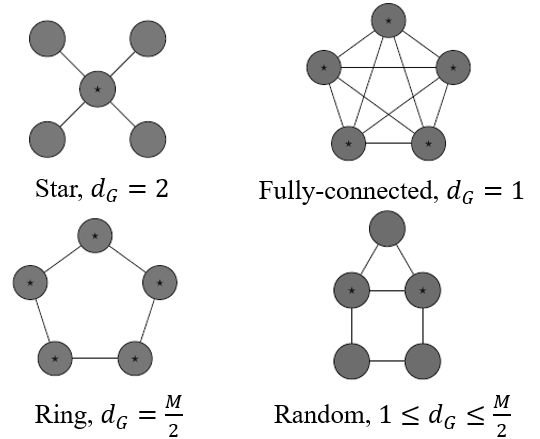}
		\caption{Examples of 5-agent graph structures.}
		\label{fig:graph_example}
	\end{minipage}\hspace{5mm}
	\begin{minipage}[t]{0.45\textwidth}
		\includegraphics[width=7cm]{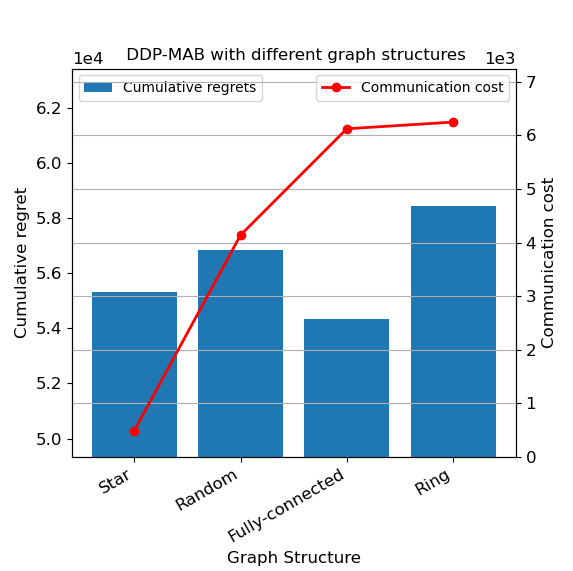}
		\caption{Regrets of DDP-MAB with different graph structures.}
		\label{fig:graph}
	\end{minipage}\hspace{5mm}

\end{figure}
\textcolor{black}{From Fig. \ref{fig:graph} we can see that the Star and Fully-connected graph achieve the smallest regrets while the Ring graph incurs larger regret. This is consistent with our results in Theorem 3. The $R_D(T)$ is equal to the $R_C(T)$ plus an item dominated by the diameter of the graph. The diameters of Fully-connected and Star graph are 1 and 2 respectively, thus their performances are close to the centralized setting. The SPIN protocol runnning in the Ring graph has the largest delay, leading to the largest regret caused by the local exploitation}. 

The communication cost is determined by both the diameter and the total number of edges. Therefore, regret and communication cost are not directly proportional. The Star graph is closest to the centralized setting and has the smallest communication cost, since the information can be completely diffused in the network as long as it is transmitted in two steps. To achieve this goal, we need to create $M(M-1)/2$ edges in the fully-connected graph, which introduces a huge communication cost.

\subsection{Hybrid Setting}
In this section, we divide 100 agents into several components with different properties. We set $c_1 = 50$ and $c_2 = 1$.

\subsubsection{Component structures}
We list the structure of each component $q$ a in Table \ref{tab:hybrid}. The sub-graphs in case1, case2 and case3 are both fully connected, however, the number of agents assigned in each component is different. Case 1, case 4 and case 5 have the same number of agents in each components while communicating over different structures of sub-graphs. In particular, case 3 can be regarded as the fully-decentralized setting where all agents are located at the vertices of one graph. Case 6 can be regarded as a centralized setting since there is only one agent in each component, which is directly connected to the server.

\begin{table}[h]
	\centering
		\caption{Different settings for hybrid structure}
	\begin{tabular}{|c|c|c|}
		\hline
		& $\{M_q \}$                & structures                   \\ \hline
		case 1 & \{20,20,20,20,20\} & fully-connected    \\ \hline
		case 2 & \{63,24,10,6,7\} & fully-connected   \\ \hline
		case 3 & \{100\} & fully-connected        \\ \hline
		case 4 & \{20,20,20,20,20\} & Random \\ \hline
		case 5 & \{20,20,20,20,20\} & Star\\ \hline
		case 6 & \{1,...,1\}$\times$ 100 & -  \\ \hline
	\end{tabular}

	\label{tab:hybrid}
\end{table}

From Fig. \ref{fig:hgraph} we can see that, although case1, case2 and case3 can finally achieve similar per-agent regrets, while the communication cost gradually increases. This is because maintaining a fully connected graph with more agents requires us to establish more agent-to-agent links. Case 1, case 4 and case 5 have a balanced agents distribution. Comparing these three case, the sub-graphs close to central setting (like star graph) achieve smaller communication cost. It also implies that the hybrid structure significantly reduces the communication cost compared with the centralized setting (case 6) or the fully-decentralized setting(case 3). 
\subsubsection{Heterogeneous Reward}
\textcolor{black}{Finally, we examine the influence of heterogeneous reward on the above three communication structures. Fig.~\ref{fig:hybrid_h} is consistent with our statements in Corollary 3. The CDP-MAB achieves the smallest regret. As we analyzed before, it can completely avoid the local bias by the coordination of the server. The other two structures are slightly affected by the heterogeneous reward. The reason is that the agents perform local exploitation during communication rounds (waiting for consensus or sink agent collection). Due to the smaller size of the sub-graph, the hybrid structure can achieve smaller delay as well as better performance than the decentralized network.}
\begin{figure}[H]
	\begin{minipage}[t]{0.45\textwidth}
		\includegraphics[width=8cm]{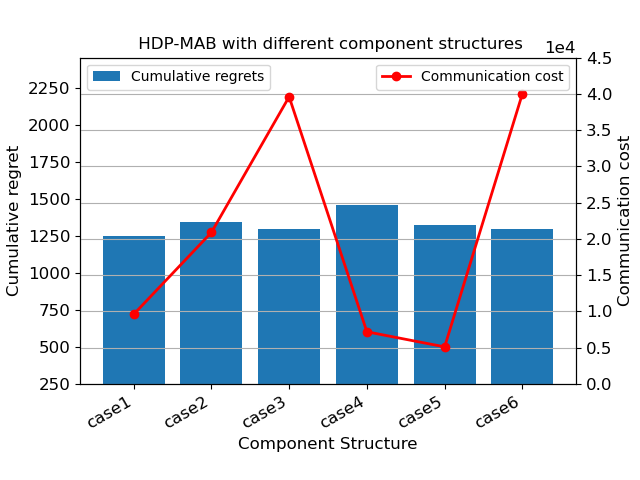}
		\caption{Regrets of HDP-MAB with different component structures.}
		\label{fig:hgraph}
	\end{minipage}\hspace{5mm}
	\begin{minipage}[t]{0.45\textwidth}
		\includegraphics[width=7.5cm]{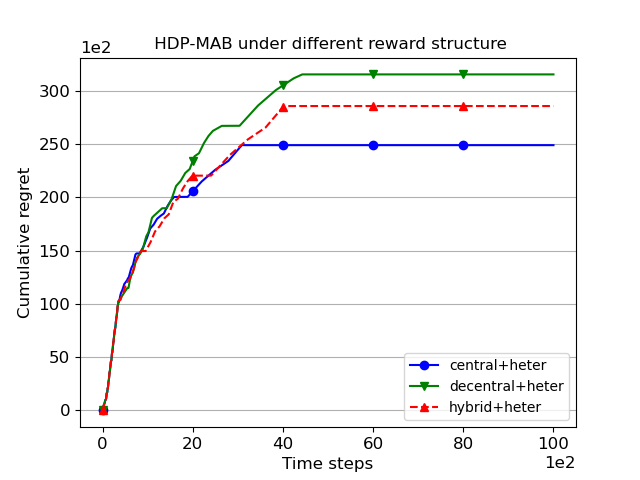}
		\caption{Regrets of HDP-MAB under different reward structures.}
		\label{fig:hybrid_h}
	\end{minipage}\hspace{5mm}
\end{figure}

\section{Conclusion}
In this paper, we propose a privacy-preserving communication-efficient framework to tackle the privacy leakage and large communication overhead issues in the FMAB problem. To protect the user privacy, we use DP techniques by adding noise before the agents send their local parameters. Theoretical results show that the DP mechanism brings a trade-off between privacy and utility. Furthermore, partial participation and less frequent communication strategies are utilized to reduce the communication cost. Decentralized structure is combined with GIS protocol to realize global information synchronization at the end of each communication round. Hybrid structure and SAC protocol are together considered to complete local aggregation inside the component first, and then implement global aggregation at the server. The above two schemes introduce an additional item related to the graph/component structure on the basis of the centralized setting regret. We compared the effectiveness of the three structures from both theoretical and experimental 
results.

\section*{Acknowledgement}
We would like to thank Prof. Christina Fragouli from UCLA for her valuable discussions and comments when we were conducting this work. We also thank Guangfeng Yan for his insightful discussion for the theoretical proofs.

\bibliography{jsacref}

\begin{thebibliography}{10}
\providecommand{\url}[1]{#1}
\csname url@samestyle\endcsname
\providecommand{\newblock}{\relax}
\providecommand{\bibinfo}[2]{#2}
\providecommand{\BIBentrySTDinterwordspacing}{\spaceskip=0pt\relax}
\providecommand{\BIBentryALTinterwordstretchfactor}{4}
\providecommand{\BIBentryALTinterwordspacing}{\spaceskip=\fontdimen2\font plus
\BIBentryALTinterwordstretchfactor\fontdimen3\font minus
  \fontdimen4\font\relax}
\providecommand{\BIBforeignlanguage}[2]{{%
\expandafter\ifx\csname l@#1\endcsname\relax
\typeout{** WARNING: IEEEtran.bst: No hyphenation pattern has been}%
\typeout{** loaded for the language `#1'. Using the pattern for}%
\typeout{** the default language instead.}%
\else
\language=\csname l@#1\endcsname
\fi
#2}}
\providecommand{\BIBdecl}{\relax}
\BIBdecl

\bibitem{li2020federated}
T.~Li, L.~Song, and C.~Fragouli, ``Federated recommendation system via
  differential privacy,'' in \emph{2020 IEEE International Symposium on
  Information Theory (ISIT)}.\hskip 1em plus 0.5em minus 0.4em\relax IEEE,
  2020, pp. 2592--2597.

\bibitem{Shi_Shen_2021}
C.~Shi and C.~Shen, ``Federated multi-armed bandits,'' \emph{Proceedings of the
  AAAI Conference on Artificial Intelligence}, vol.~35, no.~11, pp. 9603--9611,
  May 2021.

\bibitem{song2017making}
L.~Song and C.~Fragouli, ``Making recommendations bandwidth aware,'' in
  \emph{IEEE Int. Symp. Inf. Theory (ISIT)}.\hskip 1em plus 0.5em minus
  0.4em\relax IEEE, 2017, pp. 2243--2247.

\bibitem{song2018itw}
L.~Song, C.~Fragouli, and D.~Shah, ``Recommender systems over wireless:
  Challenges and opportunities,'' \emph{Proc. IEEE Inf. Theory Workshop (ITW)},
  2018.

\bibitem{8849556}
L.~{Song}, C.~{Fragouli}, and D.~{Shah}, ``Interactions between learning and
  broadcasting in wireless recommendation systems,'' in \emph{2019 IEEE
  International Symposium on Information Theory (ISIT)}, July 2019, pp.
  2549--2553.

\bibitem{lin2019uav}
Y.~Lin, T.~Wang, and S.~Wang, ``{UAV}-assisted emergency communications: An
  extended multi-armed bandit perspective,'' \emph{IEEE Communications
  Letters}, vol.~23, no.~5, pp. 938--941, 2019.

\bibitem{xu2020collaborative}
X.~Xu, M.~Tao, and C.~Shen, ``Collaborative multi-agent multi-armed bandit
  learning for small-cell caching,'' \emph{IEEE Transactions on Wireless
  Communications}, vol.~19, no.~4, pp. 2570--2585, 2020.

\bibitem{mills2019communication}
J.~Mills, J.~Hu, and G.~Min, ``Communication-efficient federated learning for
  wireless edge intelligence in {IoT},'' \emph{IEEE Internet of Things
  Journal}, vol.~7, no.~7, pp. 5986--5994, 2019.

\bibitem{haddadpour2021federated}
F.~Haddadpour, M.~M. Kamani, A.~Mokhtari, and M.~Mahdavi, ``Federated learning
  with compression: Unified analysis and sharp guarantees,'' in
  \emph{International Conference on Artificial Intelligence and
  Statistics}.\hskip 1em plus 0.5em minus 0.4em\relax PMLR, 2021, pp.
  2350--2358.

\bibitem{konevcny2016federated}
J.~Kone{\v{c}}n{\`y}, H.~B. McMahan, F.~X. Yu, P.~Richt{\'a}rik, A.~T. Suresh,
  and D.~Bacon, ``Federated learning: Strategies for improving communication
  efficiency,'' \emph{arXiv preprint arXiv:1610.05492}, 2016.

\bibitem{reisizadeh2020fedpaq}
A.~Reisizadeh, A.~Mokhtari, H.~Hassani, A.~Jadbabaie, and R.~Pedarsani,
  ``Fedpaq: A communication-efficient federated learning method with periodic
  averaging and quantization,'' in \emph{International Conference on Artificial
  Intelligence and Statistics}.\hskip 1em plus 0.5em minus 0.4em\relax PMLR,
  2020, pp. 2021--2031.

\bibitem{abadi2016deep}
M.~Abadi, A.~Chu, I.~Goodfellow, H.~B. McMahan, I.~Mironov, K.~Talwar, and
  L.~Zhang, ``Deep learning with differential privacy,'' in \emph{Proceedings
  of the 2016 ACM SIGSAC Conference on Computer and Communications
  Security}.\hskip 1em plus 0.5em minus 0.4em\relax ACM, 2016, pp. 308--318.

\bibitem{hu2020cpfed}
R.~Hu, Y.~Gong, and Y.~Guo, ``{CPFed}: Communication-efficient and
  privacy-preserving federated learning,'' \emph{arXiv preprint
  arXiv:2003.13761}, 2020.

\bibitem{li2010contextual}
L.~Li, W.~Chu, J.~Langford, and R.~E. Schapire, ``A contextual-bandit approach
  to personalized news article recommendation,'' in \emph{Proceedings of the
  19th international conference on World wide web}.\hskip 1em plus 0.5em minus
  0.4em\relax ACM, 2010, pp. 661--670.

\bibitem{zeng2016online}
C.~Zeng, Q.~Wang, S.~Mokhtari, and T.~Li, ``Online context-aware recommendation
  with time varying multi-armed bandit,'' in \emph{Proceedings of the 22nd ACM
  SIGKDD international conference on Knowledge discovery and data
  mining}.\hskip 1em plus 0.5em minus 0.4em\relax ACM, 2016, pp. 2025--2034.

\bibitem{tossou2016algorithms}
A.~C. Tossou and C.~Dimitrakakis, ``Algorithms for differentially private
  multi-armed bandits,'' in \emph{Thirtieth AAAI Conference on Artificial
  Intelligence}, 2016.

\bibitem{agarwal2017price}
N.~Agarwal and K.~Singh, ``The price of differential privacy for online
  learning,'' in \emph{International Conference on Machine Learning}.\hskip 1em
  plus 0.5em minus 0.4em\relax PMLR, 2017, pp. 32--40.

\bibitem{ren2020multi}
W.~Ren, X.~Zhou, J.~Liu, and N.~B. Shroff, ``Multi-armed bandits with local
  differential privacy,'' \emph{arXiv preprint arXiv:2007.03121}, 2020.

\bibitem{basu2019differential}
D.~Basu, C.~Dimitrakakis, and A.~Tossou, ``Differential privacy for multi-armed
  bandits: What is it and what is its cost?'' \emph{arXiv preprint
  arXiv:1905.12298}, 2019.

\bibitem{agarwal2021multi}
M.~Agarwal, V.~Aggarwal, and K.~Azizzadenesheli, ``Multi-agent multi-armed
  bandits with limited communication,'' \emph{arXiv preprint arXiv:2102.08462},
  2021.

\bibitem{zhu2021federated}
Z.~Zhu, J.~Zhu, J.~Liu, and Y.~Liu, ``Federated bandit: A gossiping approach,''
  \emph{Proceedings of the ACM on Measurement and Analysis of Computing
  Systems}, vol.~5, no.~1, pp. 1--29, 2021.

\bibitem{sankararaman2019social}
A.~Sankararaman, A.~Ganesh, and S.~Shakkottai, ``Social learning in multi agent
  multi armed bandits,'' \emph{Proceedings of the ACM on Measurement and
  Analysis of Computing Systems}, vol.~3, no.~3, pp. 1--35, 2019.

\bibitem{martinez2019decentralized}
D.~Mart{\'\i}nez-Rubio, V.~Kanade, and P.~Rebeschini, ``Decentralized
  cooperative stochastic bandits,'' in \emph{Advances in Neural Information
  Processing Systems}, 2019, pp. 4531--4542.

\bibitem{shi2021federatedpersonal}
C.~Shi, C.~Shen, and J.~Yang, ``Federated multi-armed bandits with
  personalization,'' in \emph{International Conference on Artificial
  Intelligence and Statistics}.\hskip 1em plus 0.5em minus 0.4em\relax PMLR,
  2021, pp. 2917--2925.

\bibitem{chan2011private}
T.-H.~H. Chan, E.~Shi, and D.~Song, ``Private and continual release of
  statistics,'' \emph{ACM Transactions on Information and System Security
  (TISSEC)}, vol.~14, no.~3, p.~26, 2011.

\end{thebibliography}
\appendix
\section{Appendix}
\subsection{Useful facts}
\begin{fact}[Chernoff-Hoeffding bound]
	Let $X_1,...,X_t$ be a sequence of real-valued random variables with common range $[0,1]$, and such that $\mathbb{E}[X_t|X_1,...,X_{t-1}]=\mu$. Let $S_t = \sum_{i=1}^{t}X_i$. Then for all $a\geq 0$,
	\begin{equation}
	P(S_t\geq t\mu ~ +~a)~\leq~e^{-2a^2/t},P(S_t\leq t\mu ~ -~a)~\leq~e^{-2a^2/t}\notag
	\end{equation}
\end{fact}

\begin{fact}[Concentration Bound of Laplace Distribution \cite{chan2011private}]

Let $X_1,X_2,...,X_n$ be i.i.d random variables following the $Lap(\lambda)$ distribution. Let $S_n =\sum_{i=1}^n X_i$ be the sum of $n$ variables. Then, for any $\nu\geq \lambda \sqrt{n}$ and $0<a<\frac{\sqrt{8}\nu^2}{\lambda}$, we have,
\begin{equation}
\Pr\{S_n \ge a\}<e^{\frac{a^2}{-8\nu^2}}\notag
\end{equation}

\end{fact}

\subsection{Proof of Theorem 1}
\begin{proof}
	
	Consider two streams of arm-rewards that differ on the reward of a single arm in a single timestep. This timestep plays a role in a single epoch $r$. Moreover, let a be the arm whose reward differs between the two neighboring streams. Since the reward of each arm is bounded by [0,1] it follows that the difference of the mean of arm a between the two neighboring streams is less than $\frac{1}{S(r) - S(r-1)}$. Thus, adding noise of $Lap(\frac{1}{M\epsilon [S(r) - S(r-1)]})$ to $\hat{x}_{i,k}(r)$ guarantees $M\epsilon$-DP.

	The regret incurred by Algorithm 1 can be decomposed by the local exploration of each suboptimal arm $k$ before it it is eliminated by the central server. We define $r_k$ to be the epoch up to which $\Delta_k$ exceeds $2\widetilde{\Delta}_{r_k} = 2^{-r_k+1}$. We then show that after round  $r_k$, arm $k$ will be eliminated properly with high probability. 
	
	\textbf{Step 1}: 
	We first define the event: $E_R = \{\forall k, r, |\bar{y}_k(r)-\mu_k| \leq {C(r)}\}$ for all arm $k$ in all epoch $r$ and then bound the probability it happens.  
	
		\textbf{Step 1.1}: We decompose $\bar{y}_k(r)$ as
		\begin{align}
        \bar{y}_k(r) &=\cfrac{1}{M} \sum_{i=1}^M\bar{y}_{i,k}(r) = \cfrac{1}{M}\sum_{i=1}^M \Bigg\{\frac{S(r-1)}{S(r)}\bar{y}_{i,k}(r - 1) + \frac{S(r)-S(r-1)}{S(r)}\hat{y}_{i,k}(r)\Bigg\}\nonumber\\
        &=\cfrac{1}{M}\sum_{i=1}^M\sum_{j=1}^r \frac{S(j)-S(j-1)}{S(r)} \hat{y}_{i,k}(j)\\
        &=\cfrac{1}{M}\sum_{i=1}^M\sum_{j=1}^r \frac{S(j)-S(j-1)}{S(r)} \hat{x}_{i,k}(j) + \cfrac{1}{M}\sum_{i=1}^M\sum_{j=1}^r \frac{S(j)-S(j-1)}{S(r)} l_{i,k}(j)\nonumber\\
        &\triangleq \bar{x}_k(r) + l_k(r)\nonumber
        \end{align}
		where $\bar{x}_k(r) = \frac{1}{M}\sum_{i=1}^M\sum_{j=1}^r \frac{S(j)-S(j-1)}{S(r)} \hat{x}_{i,k}(j)$, which is averaged over $MS(r)$ samples of arm $k$. $l_k(r) = \frac{1}{M}\sum_{i=1}^M\sum_{j=1}^r \frac{S(j)-S(j-1)}{S(r)} l_{i,k}(j)$, which is the accumulated noise term added on $\bar{x}_k(r)$ at the end of each epoch $r$. We then decompose the elimination threshold $C(r)$ as $C(r) = c(r) + h(r)$ . In particular, $c(r) =\sqrt{\frac{\log (8|I^{(r-1)}|r^2T)}{2MS(r)}},~ h(r) = \frac{r\sqrt{8\log(8Kr^2T)}}{M^{1.5}\epsilon S(r)}$.
		
		Therefore, for arm $k$, after epoch $r$, we have     
		\begin{align}
	\Pr\{|\bar{y}_k(r)-\mu_k| \leq {C(r)}\} 	\geq \Pr\{|\mu_k-\bar{x}_k(r)|\leq {c(r)}\} \cdot \Pr\{|l_k(r)|\leq {h(r)}\}
		\end{align}
	The first term indicates the gap between non-private aggregated estimated mean and the true unknown reward mean. According to the Hoeffding's inequality,
	\begin{align}
		\Pr\{|\bar{x}_k(r)-\mu_k| \geq{c(r)}\} \leq 2e^{{-2MS(r)c(r)^2}}
		= 2e^{{-2MS(r)}\cdot\frac{\log(8|I^{(r-1)}|r^2T)}{2MS(r)}} = \frac{1}{4|I^{(r-1)}|r^2T}
	\end{align}
	Using the union bound for all arms $k$ in $I^{(r-1)}$ and for all epoch $r$, we have
	\begin{equation}
        \Pr\{\forall k, r, |\bar{x}_k(r)-\mu_k| \leq{c(r)}\} \geq 1 - \frac{1}{2T}
    \label{eq:condition1}
	\end{equation} 
	The second term represents the accumulated noise added on $\bar{x}_k(r)$. Since the noise is generated using the Laplace mechanism, we use the concentration property of the Laplace distribution to bound the this term. Specifically, for $l_{i,k}(j)\sim Lap(\frac{1}{M\epsilon [S(j) - S(j-1)]})$, setting $\nu = \frac{\sqrt{M}}{M\epsilon [S(j) - S(j-1)]}$ for any $0<a<\frac{\sqrt{8}\nu^2}{\lambda}$, we want
	
\begin{align}
\Pr\{|\frac{1}{M}\sum_{i=1}^M l_{i,k}(j)|\geq a\}\leq 2e^{\frac{M^2a^2M^2\epsilon^2[S(j) - S(j-1)]^2}{-8M}}  = e^{-\log (4|I^{(j-1)}|j^2T)}= \frac{1}{4|I^{(j-1)}|j^2T}
\end{align}
By solving above equation, we can get $a=\frac{\sqrt{8\log(8|I^{(j-1)}|j^2T)}}{M^{1.5}\epsilon [S(j) - S(j-1)]}$. That is,   

\begin{align}
&\Pr\Big\{|\frac{\sum_{i=1}^M l_{i,k}(j)}{M}|\geq \frac{\sqrt{8/M\log(8|I^{(j-1)}|j^2T)}}{M\epsilon [S(j) - S(j-1)]}\Big\} \le \frac{1}{4|I^{(j-1)}|j^2T}
\end{align}

Using the union bound for all arms $k$ in $I^{(r-1)}$ and for all epoch $r$, we have
\begin{align}
&\Pr\Big\{\exists k, r, |\frac{\sum_{i=1}^M l_{i,k}(r)}{M}|\geq \frac{\sqrt{8/M\log(8|I^{(r-1)}|r^2T)}}{M\epsilon [S(r) - S(r-1)]}\Big\} \le \frac{1}{2T}
\end{align}

Thus, we have probability at least $1-\frac{1}{2T}$, for all arms $k$ in $I^{(r-1)}$ and for all epoch $r$
\begin{align}
    |l_k(r)| &= \cfrac{1}{M}\sum_{i=1}^M\sum_{j=1}^r \frac{S(j)-S(j-1)}{S(r)} |l_{i,k}(j)| \notag\\
    &\le \sum_{j=1}^r\frac{S(j)-S(j-1)}{S(r)}\frac{\sqrt{8\log(8|I^{(j-1)}|j^2T)}}{M^{1.5}\epsilon [S(j) - S(j-1)]} \notag\\
    &\le \frac{r\sqrt{8\log(8Kr^2T)}}{M^{1.5}\epsilon S(r)} \le h(r) \label{eq:condition2}
\end{align}

Combing Eq.~\eqref{eq:condition1} and Eq.~\eqref{eq:condition2}, we have $\Pr\{E_R\} \geq 1- \frac{1}{T}$.
	
\textbf{Step 2}: 	
 We continues the proof under the assumption that $E_R$ holds. 
 
 \textbf{Step 2.1}: We first argue that the optimal arm $1$ is never eliminated  if $E_R$ holds. Indeed, For any epoch $r$ and any arm $k$ in the epoch we have $|\bar{y}_k(r)-\mu_k|\leq c(r)+h(r)$. 	Denote $a^*(r)$ be the arm has the highest private empirical mean in current epoch $r$.	It is easy to verify that $\bar{y}_{a^*(r)}(r)-\bar{y}_1(r)\leq 2C(r)$ since
 \begin{align}
 \bar{y}_{a^*(r)}(r)-\bar{y}_1(r) &\le \bar{y}_{a^*(r)}(r)-\bar{y}_1(r) - (\mu_{a^*(r)} - \mu_1)\notag\\
& \le |(\bar{y}_{a^*(r)}(r) - \mu_{a^*(r)}) - (\bar{y}_1(r) - \mu_1)|\notag\\
& \le |(\bar{y}_{a^*(r)}(r) - \mu_{a^*(r)})| + |(\bar{y}_1(r) - \mu_1)|\notag\\
& \le  2(c(r)+h(r))
 \end{align}
 Thus, the optimal arm is never eliminated. 
 
  \textbf{Step 2.2}: We next argue that if $E_R$ holds, in epoch $r_k$, the algorithm eliminates suboptimal arm $k$ with gap $\Delta_k$ larger than $2\widetilde{\Delta}_{r_k}$. With the choices of $c(r)$,  $h(r)$ and $S(r)$, we have,

	\begin{equation}
\begin{array}{lllll}
&c(r_k) =\sqrt{\frac{\log (8|I^{(r_k-1)}|r_k^2T)}{2MS(r_k)}} \leq \sqrt{ \frac{\log (8|I^{(r_k-1)}|r_k^2T)}{2M\cdot \frac{8\log(8|I^{(r_k-1)}|r_k^2T)}{M\widetilde{\Delta}_{r_k}^2}}}= \widetilde{\Delta}_{r_k}/4 \notag\\

&h(r_k) = \frac{r_k\sqrt{8/M\log(8Kr_k^2T)}}{M\epsilon S(r_k)} \leq \frac{r_k\sqrt{8/M\log(8Kr_k^2T)}}{M\epsilon \frac{8r_k \sqrt{2/M\log(8Kr_k^2T)} }{M\epsilon\widetilde{\Delta}_{r_k}}} = \widetilde{\Delta}_{r_k}/4 
\end{array}
\end{equation}
 Thus, $c(r_k)+h(r_k)\leq \widetilde{\Delta}_{r_k}/2$. So for arm $k$,
\begin{align}
&\bar{y}_k(r_k) + ( c(r_k)+h(r_k)) \overset{(a)}{\leq} \mu_k + 2(c(r_k)+h(r_k)) \notag\overset{(b)}{\leq} \mu_k+\Delta_k - 2(c(r_k)+h(r_k)) \notag\\&\overset{(c)}{=} \mu_1 - 2(c(r_k)+h(r_k))\overset{(d)}{\leq}\bar{y}_1(r_k)  - (c(r_k)+h(r_k)) \leq \bar{y}_{a^*(r_k)}(r_k)  - (c(r_k)+h(r_k))
\end{align}
That is $\bar{y}_k(r_k) \leq \bar{y}_{a^*(r_k)}(r_k)  - 2(c(r_k)+h(r_k)) = \bar{y}_{a^*(r_k)}(r_k)  - 2C(r_k)$, which guarantees that suboptimal arm $k$ is eliminated after round $r_k$. 
In Eq.(19), (a) and (d) use the condition that $E_R$ holds; (b) uses the elimination threshold that $\Delta_k\geq 2\widetilde{\Delta}_{r_k}\geq 4(c(r_k)+h(r_k))$. And (c) is from the fact $\mu_k = \mu_1-\Delta_k$.

  \textbf{Step 3}:  We conclude by computing the total number of arms pulls $n_{i,k}(T)$ required for each suboptimal arm $k$ at each agent $i$. Specifically, arm $k\neq 1$ does not survive round $r_k$ with $ r_k = \left \lceil \log (\frac{1}{\Delta_k}) +1 \right \rceil $ since $\Delta_k \geq 2\widetilde{\Delta}_{r_k}$.
\begin{align}
&n_{i,k}(T)\leq S(r_k) \notag\\
&\leq\max\{\frac{8\log(8|I^{(r_k-1)}|r^2T)}{M\widetilde{\Delta}_{r_k}^2 }, \frac{8r_k \sqrt{2\log(8Kr_k^2T)} }{M^{1.5}\epsilon\widetilde{\Delta}_{r_k}}\}\notag\\
&\leq \max\{\frac{16\log(8K (\left \lceil \log (\frac{1}{\Delta_k}) +1 \right \rceil)^2 T)}{M\Delta_k^2}, \frac{32\left \lceil \log (\frac{1}{\Delta_k}) +1 \right \rceil)\sqrt {2\log (8KT (\left \lceil \log (\frac{1}{\Delta_k}) +1 \right \rceil)^2  )} }{M^{1.5}\epsilon\Delta_k}\notag\\
&=O(\max\{\frac{\log(KT\log T)}{M\Delta_k^2}, \frac{\log T \sqrt{\log(KT\log T)}}{M^{1.5}\epsilon\Delta_k}\})
\end{align} 
The last equality is due to the fact that $\left \lceil \log (\frac{1}{\Delta_k}) +1 \right \rceil < \log T$.
When $E_R$ does not hold, the maximum regret is $\Delta_{max}\cdot T\cdot \Pr\{\bar{E_R}\} \leq \Delta_{max}\leq 1$. Therefore, we only consider the regret when $E_R$ holds with probability $1-1/T$. Summing up all $M$ agents and $K$ arms we can conclude,
\begin{align}
&R_{C}(T) = \sum_{i=1}^M\sum_{k=1}^{K}\Delta_k \cdot n_{i,k}(T) = O(\max\{\sum_{k=1}^{K} \frac{\log(KT\log T)}{\Delta_k}, \frac{K\log T \sqrt{\log(KT\log T)}}{\sqrt{M}\epsilon}\})
\end{align}

\end{proof}

\subsection{Proof of Theorem 2}

\begin{proof}
The communication cost $C^{p,R}_C(T)$ can be directly derived from $C_C(T)$ by replacing the participants $M$ as $pM$ and required communication round $O(\log T)$ as fixed $R$. 

We use the similar techniques in Theorem 1 to investigate the regret. The regret incurred by Algorithm 1 can be decomposed by the local exploration of each suboptimal arm $k$ before it it is eliminated by the central server. We define $r_k$ to be the epoch up to which $\Delta_k$ exceeds $2\widetilde{\Delta}_{r_k} = 2 \Delta^{r/R}$. We then show that after round  $r_k$, arm $k$ will be eliminated properly with high probability. Notice the $\widetilde{\Delta}_{r_k} $  used in Theorem 1 is doubling-decreasing, while it is exponentially decreasing with scale $R$ in the modified algorithm. This directly leads to a different $r_k$ required for eliminate arm $k$.

	\textbf{Step 1}: Compared with $S(r)$ set in Alg.1, the  $S^p(r)$ has a scale factor of $\frac{2^{-r}}{p\Delta^{r/R}}$ or $\frac{2^{-r}}{p^{3/2}\Delta^{r/R}}$ for the two terms inside the $\max\{\}$. 
	Since the value of $C^p(r)$ is determined by $S^p(r)$, it is scaled equally with respect to $C(r)$ set in Alg.1. This means that the events the event $E_R = \{\forall k, r, |\bar{y}_k(r)-\mu_k| \leq {C^p(r)}\}$ for all arm $k$ in all epoch $r$ still holds with probability $1-\frac{1}{T}$, where $\bar{y}_k(r) = \frac{1}{N}\sum_{i=1}^N \bar{y}_{i,k}(r)$ is the empirical averaged mean aggregate from $N$ participants.

	\textbf{Step 2}: 	
If $E_R$ holds, in any epoch $r$, two of the following events happens:\\
1) the optimal arm always remains in epoch $r$;\\
2) the algorithm eliminates all suboptimal arms with gap $\Delta_k$ larger than $2\widetilde{\Delta}_{r_k}=2 \Delta^{r/R}$. This also demonstrates that the second best arm with gap $\Delta$ is removed from the active arm set after the $R-$th communication since $\widetilde{\Delta}_R = \Delta^{R/R} = \Delta$. 

We conclude by computing the total number of arms pulls $n_{i,k}(T)$ required for each suboptimal arm $k$ at each agent $i$. Specifically, arm $k\neq 1$ does not survive round $r_k$ with,
	
	\begin{align}
	r_k = \left \lceil  \frac{R(1+\log(\frac{1}{\Delta_k}))}{\log(\frac{1}{\Delta})}\right \rceil 
	\end{align}
Since arm $k$ is still in the active arm set at the end of round $r_k-1$, we have \begin{equation}
\Delta_k < 2\widetilde{\Delta}_{r_k-1} = 2\Delta^{\frac{r_k-1}{R}} = \frac{2\Delta^{\frac{r_k}{R}}}{\Delta^{1/R}} =\frac{2\widetilde{\Delta}_{r_k}}{\Delta^{1/R}}    
\end{equation}

Using $2\widetilde{\Delta}_{r_k} \geq \Delta^{1/R}\cdot \Delta_k$, we can upper bound the number of times arm $k$ was pulled by agent $i$ as: 
	\begin{align}
&n_{i,k}(T) \leq S^p_k(r_k) \leq \frac{1}{M}\max \left\lbrace  \frac{8\log(8|I^{(r_k-1)}|r_k^2T)}{p\widetilde{\Delta}_{r_k}^2 }, \frac{8r_k \sqrt{2\log(8Kr_k^2T)} }{\sqrt{M}p^{3/2}\epsilon\widetilde{\Delta}_{r_k}}\right\rbrace  \notag \\
& \leq \frac{1}{M}\max \left\lbrace  
\frac{96 \log(r_k^2KT)}{p{\Delta}_{k}^2\cdot \Delta^{2/R}}, \frac{16r_k \sqrt{6\log(r_k^2KT)}}{\sqrt{M}p^{3/2}\epsilon{\Delta}_{k}\cdot \Delta^{1/R}}	\right\rbrace  \notag \\
& \leq \frac{1}{M}\max \left\lbrace  
\frac{96 \log(R^2KT)}{p{\Delta}_{k}^2\cdot \Delta^{2/R}}, \frac{16R \sqrt{6\log(R^2KT)}}{\sqrt{M}p^{3/2}\epsilon{\Delta}_{k}\cdot \Delta^{1/R}}	\right\rbrace  \notag \\
&= O(\frac{1}{M}\max \left\lbrace  \frac{\Delta^{-2/R}}{p\Delta_k^2} \log (R^2KT),\frac{ \Delta^{-1/R}}{\sqrt{M}p^{3/2}\epsilon \Delta_k} R\sqrt{\log (R^2KT)}\right\rbrace )\notag \\
& \leq O(\frac{1}{M}\max \left\lbrace  \frac{\log (R^2KT)}{p\Delta_k^2} T^{2/R},\frac{R\sqrt{\log (R^2KT)}}{\sqrt{M}p^{3/2}\epsilon \Delta_k}T^{1/R} \right\rbrace )\notag \\
& \leq O(\frac{T^{2/R}}{p^{3/2}M}\max \left\lbrace  \frac{\log (R^2KT)}{\Delta_k^2},\frac{R\sqrt{\log (R^2KT)}}{\sqrt{M}\epsilon \Delta_k} \right\rbrace )
	\end{align}
The second last equality is due to the fact $\Delta = O(\frac{1}{T})$. If $\Delta < \frac{1}{T}$, even if we play suboptimal arms for all $T$ slots, we can only incur a regret less than $1$. The last equality is due to $T^{2/R} > T^{1/R}$ and $p>p^{3/2}$.	
 	
Summing up all $M$ agents and $K$ arms, the regret is upper bounded by,
	\begin{align}
\sum_{i=1}^M\sum_{k=1}^{K}\Delta_k \cdot n_{i,k}(T) = O (\frac{T^{2/R}}{p^{3/2}}\max \left\lbrace  \frac{\log (R^2KT)}{\Delta_k},\frac{R\sqrt{\log (R^2KT)}}{\sqrt{M}\epsilon} \right\rbrace)
	\end{align}
 In order to make our performance not worse than the non-communication case, we need to ensure that $T^{2/R} /p\leq  M\rightarrow R \geq \log T /\log pM$.	
We can conclude the cumulative regret of the modified algorithm by setting the $\min\{M,\frac{T^{1/R}}{p^{3/2}}\}$ operation.
\end{proof}

\subsection{Proof of Theorem 3}

We first investigate $C_{D}(T)$. According to our definition of communication cost, it is determined by the number of connections established in all communication rounds. 

We first analyze the total number of required communication rounds.
Consider that $S(r)$ and $C(r)$ we set in Alg.2 are the same as Alg.1, and the GIS protocol can ensure that each agent can aggregate the same average mean as the server in Alg.1 after communication round $r$. So we can directly use the conclusion in Theorem 1, that is, all agents can synchronously identify the best arm after $\left \lceil \log (\frac{1}{\Delta_k}) +1 \right \rceil $ rounds.

We next examine the number of connections established by each communication round. First, the duration of each round is $t_{\text{delay}}$ time slots. In each time slot, the upper bound of connections is the number of edges in the graph $\sum_{i=1}^M d_i/2$. It is also known that the max $t_{\text{delay}}$ is the diameter of the graph $d_G$, which is the time delay for the two vertices farthest apart on the graph to receive information from each other. Thus, the upper bound of connections established by each communication round is $d_G \sum_{i=1}^M d_i/2$. Multiply this with the total required communication round we can conclude $C_{D}(T)$.

We then calculate the regret $R_{D}(T) $, which can be further divided into two terms as:
\begin{equation}\label{regret5}
O(\underbrace{\sum_{k=1}^{K} \frac{\log(KT\log T)}{\Delta_k}, \frac{K\log T \sqrt{\log(KT\log T)}}{\sqrt{M}\epsilon}}_{(1)}+\underbrace{M(d_G-1)}_{(2)})
 	\end{equation}	
 
\textbf{Step 1}: $R_{D(1)}(T)$ is caused by local explorations of suboptimal arms. Each agent is evenly allocated $ S(r)$ times on each $k\in I_i^{(r-1)}$. Note our GIS protocol ensures that all agents can observe totally $S(r)$ samples at the end of round $r$. Thus, although the agents perform aggregation and elimination independently, they can generate the same aggregated means $\bar{y}_{i,k}(r)$ and the same active arm set $I_i^{(r)}$ for all $i\in[M], k\in[K]$. The advantage of this is to ensure that no additional asynchronous delay is introduced in the next exploration phases. In this way, each agent can be seen as a “central server” in Alg.1 connected with other $M-1$ agents. With the same selection $S(r)$ and $C(r)$, following the analysis of Theorem 1, we can conclude that $R_{D(1)}(T) = R_{C}(T)$. Therefore, the impact of decentralized setting is mainly reflected in the second item.

\textbf{Step 2}: $R_{D(2)}(T)$ is incurred by inappropriate exploitation on suboptimal arms during communication. Before the end of communication round $r$, each agent can only greedily exploit empirically best arm based on local observations in this round. Therefore, The regret introduced at this stage is related to the accuracy of the estimation and $t_{delay}$. 

\textbf{Step 2.1}: The largest $t_{delay}$ will not exceed the diameter of the graph $(d_{G}-1)$, because this is the time slots that the two farthest nodes on the graph need to pass through to receive the observations from each other.

\textbf{Step 2.2}: Next, we examine the regrets introduced in each slot of $t_{delay}$ when exploit suboptimal arm. Recall the proof in Theorem 1, when event $E_R$ holds, suboptimal arm with $\Delta_k>2^{-(r-1)}$ will be eliminated after round $r-1$ (Step 2.2) and the optimal arm is never eliminated (Step 2.1). Therefore, each pull can cause at most $2^{-(r-1)}$ regret and each communication round incurs at most $M(d_{G}-1)2^{-(r-1)}$ regret. Summing up all required round we obtain: 
\begin{eqnarray}
&~&M(d_{G}-1)\sum_{r= 1}^{\left\lceil\log (\frac{1}{\Delta})+1\right\rceil}(2^{-(r-1)}) = M(d_{G}-1) (1+ 1/2+1/4 +...+ 2^{-\log (\frac{1}{\Delta})})\notag\\
&~&=2M(d_{G}-1) (1- 2^{-\log (\frac{1}{\Delta})})=2M(d_{G}-1) (1- \Delta)
\end{eqnarray}
If event $E_R$ does not hold, the max regret that can be caused by each pull is 1.
Then at most $Md_G$ regret will be introduced at round $r$ and $Md_G\left\lceil\log (\frac{1}{\Delta})+1\right\rceil$ will be incurred by all communication rounds. We can bound the $R_{D(2)}$ by: 
\begin{eqnarray}
&~&R_{(2)} = \Pr\{E_R\} (2M(d_{G}-1) (1- \Delta)) + \Pr\{\bar{E}_R\} (M(d_{G}-1))\notag\\
&~&\leq (1-\frac{1}{T}) (2M(d_{G}-1) (1- \Delta)) + (\frac{1}{T}) M(d_{G}-1)\left\lceil\log (\frac{1}{\Delta})+1\right\rceil\notag \\ 
&~&\leq M(d_{G}-1)(2+ \frac{\log T}{T})= O(M(d_{G}-1))
\end{eqnarray}
The last equality is due to the facts $\log T/T \rightarrow 0$ with large $T$. We complete the proof by summing up $R_{D(1)} (T)+ R_{D(2)}(T)$.

\subsection{Proof of corollary 1}

\begin{proof}
The communication cost $C^R_{D}(T)$ can be directly derived $C_{D}(T)$ by replacing the required communication rounds as $R$.

We then calculate the regret $R^R_{D}(T) $, which can be further divided into two terms as:
\begin{equation}\label{regret5}
O(\underbrace{\min\{M,{T^{2/R}}\}\cdot \max \left\lbrace  \sum_{k=1}^{K}\frac{\log (R^2KT)}{\Delta_k} ,\frac{ R\sqrt{\log (R^2KT)}}{\sqrt{M}\epsilon} \right\rbrace  }_{(1)}  + \underbrace{M(d_{G}-1)}_{(2)})
 	\end{equation}	
 
The first term is caused by local explorations of all $M$ agents, which recover the result of Alg.1 $R^R_{D(1)}(T) = R^{p,R}_{C}(T)$ with $p = 1$.

 Next, we examine the second term introduced in each slot of $t_{delay}$ when exploit suboptimal arm. Recall the proof in Theorem 2, when event $E_R$ holds, suboptimal arm with $\Delta_k>2\widetilde{\Delta}_{r_k}=2 \Delta^{r/R}$ will be eliminated after round $r-1$ and the optimal arm is never eliminated (Step 2). Therefore, each pull can cause at most $\Delta^{(r-1)/R}$ regret and each communication round incurs at most $Md_G\Delta^{(r-1)/R}$ regret. Summing up all required round we obtain: 
\begin{eqnarray}
&~&M(d_{G}-1)\sum_{r= 1}^{R}(\Delta^{r/R}) =  M(d_{G}-1) (1+ \Delta^{1/R}+\Delta^{2/R}+...+\Delta^{R/R})\notag\\
&~&=M(d_{G}-1) (\frac{1-\Delta}{1-\Delta^{1/R}})
\end{eqnarray}
If event $E_R$ does not hold, the max regret that can be caused by each pull is 1.
Then at most $Md_G$ regret will be introduced at round $r$ and $Md_GR$ will be incurred by all communication rounds. We can bound the $R^R_{D(2)}$ by: 
\begin{eqnarray}
&~&R^R_{D(2)}(T) \leq \Pr\{E_R\} (M(d_{G}-1) (\frac{1-\Delta}{1-\Delta^{1/R}})) + \Pr\{\bar{E}_R\} (M(d_{G}-1) R)\notag\\
&~&\leq M(d_{G}-1) (\frac{1-\Delta}{1-\Delta^{1/R}}+ \frac{R}{T}) = O(M(d_{G}-1))
\end{eqnarray}
The last equality is due to the facts: i) $\frac{1-\Delta}{1-\Delta^{1/R}}+ \frac{R}{T}\leq 1$ since $R\geq 1$; ii) $R$ is no larger than $T$.  We complete the proof by summing above two terms. 

\end{proof}

\subsection{Proof of Theorem 4}

\begin{proof}
We first investigate the communication cost. In the local communication of $q$, the connection between $i$ and $SA_q$ holds only when $\hat{y}_{i,k}(r)$ is sent from $i$ to $SA_q$, after that, the connection broken. In the global communication, at most $Q$ connections are required. Therefore the cost in each round is $c_2\sum_{q \in Q}\sum_{i\in [M_q]} sd(i,SA_q) + c_1Q $ where $sd(i,j)$ is the shortest distance between agent $i$ and $j$. Summing all communication round before we identify the best arm, the communication cost $$C_{H}(T) = (c_2\sum_{q \in Q}\sum_{i\in [M_q]} sd(i,SA_q) + c_1Q) \left\lceil\log (\frac{1}{\Delta})+1\right\rceil$$ We conclude the result with the fact  $sd(i,SA_q) \leq d_G^q $ and $\left\lceil\log (\frac{1}{\Delta})+1\right\rceil <\log T$.

The proof of regret is similar as Theorem 3. We also divide the regret as $R_{H(1)}(T)$ and $R_{H(2)}(T)$. The first term is caused by local explorations of all $M$ agents, which recover the result of Alg.1 $R_{H(1)}(T) = R_{C}(T)$. The second term is incurred by greedily pulling the empirical best arm during $t_{\text{delay}}$, which is determined by the slowest component that finish the sink agent collection:$t_{\text{delay}} =  \max_{q\in Q}\{\max_{i\in [M_q]}\{sd(i,SA_q)\}\}$. Obviously, the largest $t_{delay}$ will not exceed the largest diameter of the all graphs $\max_{q\in Q}\{d_G^q-1\}$. Next, we follow Step 2.2 
 to examine the regrets introduced in each slot of $t_{delay}$ when exploit suboptimal arm. Recall the proof in Theorem 1, when event $E_R$ holds, suboptimal arm with $\Delta_k>2^{-(r-1)}$ will be eliminated after round $r-1$ (Step 2.2) and the optimal arm is never eliminated (Step 2.1). Therefore, the incurred regret is upper bounded by: 
\begin{eqnarray}
&~&M\max_{q\in Q}\{d_G^q-1\}\sum_{r= 1}^{\left\lceil\log (\frac{1}{\Delta})+1\right\rceil}(2^{-(r-1)}) = M\max_{q\in Q}\{d_G^q-1\} (1+ 1/2+1/4 +...+ 2^{-\log (\frac{1}{\Delta})})\notag\\
&~&=2M\max_{q\in Q}\{d_G^q-1\} (1- 2^{-\log (\frac{1}{\Delta})})=2M\max_{q\in Q}\{d_G^q-1\} (1- \Delta)
\end{eqnarray}
If event $E_R$ does not hold, the max regret that can be caused by each pull is 1. We have,
\begin{eqnarray}
&~&R_{H(2)}(T)\leq \Pr\{E_R\} (2M\max_{q\in Q}\{d_G^q-1\} (1- \Delta)) + \Pr\{\bar{E}_R\} (M\max_{q\in Q}\{d_G^q-1\})\notag\\
&~&\leq (1-\frac{1}{T}) (2M\max_{q\in Q}\{d_G^q-1\} (1- \Delta)) + (\frac{1}{T}) M\max_{q\in Q}\{d_G^q-1\}\left\lceil\log (\frac{1}{\Delta})+1\right\rceil\notag \\ 
&~&\leq M \max_{q\in Q}\{d_G^q-1\} (2+ \frac{\log T}{T}) = O(Md_G \max_{q\in Q}\{d_G^q-1\})
\end{eqnarray}
The last equality is due to the fact $\frac{\log T}{T}\rightarrow 0$ with large $T$. We complete the proof by summing $R_{H(1)}(T)$ and $R_{H(2)}(T)$.

\end{proof}

\end{document}